\documentclass[final,12pt]{arxiv}

\usepackage{microtype}
\usepackage{graphicx}
\usepackage{booktabs} 
\usepackage{hyperref}

\usepackage{algorithm}
\usepackage{algorithmic}

\usepackage{amsmath}
\usepackage{amssymb}

\usepackage[capitalize,noabbrev]{cleveref}

\usepackage{dsfont}
\usepackage[mathscr]{euscript}
\usepackage{xcolor}
\usepackage{colortbl}
\usepackage{thmtools}
\usepackage{thm-restate}

\usepackage{mathtools}

\usepackage{multirow}
\usepackage{wrapfig}

\usepackage{pifont}

\usepackage{MnSymbol}
\DeclareMathAlphabet\mathbb{U}{msb}{m}{n}
\usepackage{xpatch}

\def\Rset{\mathbb{R}}

\let\Pr\undefined

\DeclareMathOperator*{\Pr}{\mathbb{P}}

\DeclareMathOperator*{\E}{\mathbb E}
\DeclareMathOperator*{\argmax}{argmax}
\DeclareMathOperator*{\argmin}{argmin}

\DeclareMathOperator{\Reg}{\mathsf{Reg}}

\DeclareMathOperator*{\Var}{Var}
\DeclareMathOperator{\Ind}{\mathbb{I}} 

\DeclarePairedDelimiter{\abs}{\lvert}{\rvert} 
\DeclarePairedDelimiter{\bracket}{[}{]}
\DeclarePairedDelimiter{\curl}{\{}{\}}
\DeclarePairedDelimiter{\paren}{(}{)}
\DeclarePairedDelimiter{\norm}{\|}{\|}

\newcommand{\cB}{\mathcal{B}}

\newcommand{\cF}{\mathcal{F}}

\newcommand{\cN}{\mathcal{N}}

\newcommand{\sD}{{\mathscr D}}
\newcommand{\sE}{{\mathscr E}}

\newcommand{\sG}{{\mathscr G}}
\newcommand{\sH}{{\mathscr H}}

\newcommand{\sR}{{\mathscr R}}

\newcommand{\sX}{{\mathscr X}}
\newcommand{\sY}{{\mathscr Y}}

\newcommand{\bc}{{\mathbf c}}

\newcommand{\bw}{{\mathbf w}}

\newcommand{\bone}{\mathbf{1}}

\newcommand{\sfL}{{\mathsf L}}

\newcommand{\1}{\mathds{1}}

\newcommand{\hh}{{\sf h}}
\newcommand{\rr}{{\sf r}}

\newcommand{\ov}{\overline}

\newcommand{\wt}{\widetilde}
\newcommand{\e}{\epsilon}
\newcommand{\ignore}[1]{}

\newcommand{\ldef}{{\sfL_{\rm{def}}}}
\newcommand{\tdef}{{\sfL_{\rm{tdef}}}}
\newcommand{\num}{{n_e}}
\newcommand{\expert}{{g}}
\newcommand{\expertexpert}{{\sf g}}
\newcommand{\sur}{\sfL}
\newcommand{\var}{\text{var}}

\usepackage{array}
\newcolumntype{H}{>{\setbox0=\hbox\bgroup}c<{\egroup}@{}}

\hypersetup{
  breaklinks   = true, 
  colorlinks   = true, 
  urlcolor     = blue, 
  linkcolor    = blue, 
  citecolor   = blue 
}

\usepackage[toc, page, header]{appendix}
\usepackage{times}
\title[Budgeted Multiple-Expert Deferral]{Budgeted Multiple-Expert Deferral}

\arxivauthor{%
 \Name{Giulia DeSalvo} \Email{giuliad@google.com}\\
 \addr Google DeepMind, Seattle%
 \AND
 \Name{Clara Mohri} \Email{cmohri@g.harvard.edu}\\
 \addr Harvard University, Cambridge%
 \AND
 \Name{Mehryar Mohri} \Email{mohri@google.com}\\
 \addr Google Research and Courant Institute of Mathematical Sciences, New York%
 \AND
 \Name{Yutao Zhong} \Email{yutaozhong@google.com}\\
 \addr Google Research, New York%
}

\begin{document}

\maketitle

\begin{abstract}
Learning to defer uncertain predictions to costly experts offers a
powerful strategy for improving the accuracy and efficiency of machine
learning systems. However, standard training procedures for deferral
algorithms typically require querying all experts for every training
instance, an approach that becomes prohibitively expensive when expert
queries incur significant computational or resource costs. This
undermines the core goal of deferral: to limit unnecessary expert
usage. To overcome this challenge, we introduce the \emph{budgeted
deferral} framework, which aims to train effective deferral algorithms
while minimizing expert query costs during training.
We propose new algorithms for both two-stage and single-stage
multiple-expert deferral settings that selectively query only a subset of
experts per training example. While inspired by active learning, our
setting is fundamentally different: labels are already known, and the
core challenge is to decide which experts to query in order to balance
cost and predictive performance. We establish theoretical guarantees
for both of our algorithms, including generalization bounds and label
complexity analyses.
Empirical results across several domains show that our algorithms
substantially reduce training costs without sacrificing prediction
accuracy, demonstrating the practical value of our budget-aware
deferral algorithms.
\end{abstract}



\section{Introduction}

Learning algorithms can improve accuracy and efficiency by deferring
uncertain predictions to experts, such as domain specialists or
advanced pre-trained models. Examples include biomedical diagnosis with radiologists of different specialties and geological monitoring with sensor networks of varying capabilities. To do this effectively, it is essential
to account for the cost associated with each expert, which may
represent prediction quality, latency, or other computational aspects. These
costs may vary with each input instance and depend on the possible
output labels.

The problem of assigning each input to the most appropriate
expert, balancing predictive accuracy with resource consumption is
known as \emph{learning to defer with multiple experts}. This task
arises in many domains, including natural language generation with
large language models \citep{WeiEtAl2022, bubeck2023sparks}, speech
recognition, image classification, financial
forecasting, and computer vision.

Recent work has extensively studied this problem in the context of
classification \citep{hemmer2022forming, keswani2021towards,
  kerrigan2021combining, straitouri2022provably,
  benz2022counterfactual, verma2023learning,
  MaoMohriMohriZhong2023two, MaoMohriZhong2024deferral}, and more
recently in regression \citep{mao2024regression} and multi-task
learning settings \citep{montreuil2025two}. Several algorithms with
strong theoretical results have been developed, including surrogate
loss-based approaches for multiple-expert deferral with several
consistency guarantees \citep{MaoMohriZhong2025}.

These surrogate loss-based approaches can be broadly categorized into two
settings \citep{mao2025theory}. In the \emph{single-stage} setting, the predictor and
deferral function are learned jointly \citep{mozannar2020consistent,
  verma2022calibrated, charusaie2022sample, pmlr-v206-mozannar23a,
  MaoMohriZhong2024deferral, maorealizable}. In contrast, the
two-stage setting first trains a predictor, which is then held fixed
and treated as an additional expert, while the deferral function is
learned in a subsequent stage \citep{MaoMohriMohriZhong2023two}. See
Appendix~\ref{app:related-work} for further discussion.

However, despite extensive prior work, a central challenge remains: in
many real-world applications, especially those involving costly
experts such as LLMs, training a deferral algorithm can itself be
computationally prohibitive. For each labeled example $(x, y)$, the
outputs of all $p$ experts must be computed, incurring the total cost
of all experts, multiplied across the full training set. This is at
odds with the core motivation of deferral: to reduce unnecessary
expert usage.

Can we train effective deferral algorithms while reducing the
computational burden? One natural idea is to selectively query only
a subset of experts for each training instance, thereby lowering
training cost. However, this introduces partial information and raises
new questions: How should we choose which experts to query? Can we do
so in a principled way that maintains performance guarantees for the
deferral algorithm, while controlling computational cost?

This paper addresses these questions and introduces a formal study of
this \emph{budgeted deferral problem}, where the goal is to train
deferral algorithms while minimizing the cost of querying experts. We
propose new algorithmic solutions for both the two-stage
multiple-expert deferral setting \citep{MaoMohriMohriZhong2023two} and
the single-stage multiple-expert deferral setting
\citep{verma2023learning,MaoMohriZhong2024deferral}. While our methods
draw inspiration from active learning \citep{beygelzimer2009importance,cortes2019disgraph,cortes2019rbal,cortes2020adaptive}, the setup is fundamentally
different: unlike standard active learning, our training data consists
of labeled examples, that is, pairs $(x, y)$ are already known. The
challenge is not whether to request a label, but rather which experts
to query for each example, in order to balance cost and predictive
performance.

\textbf{Contributions.} 
Our main contributions are as follows, and they also define the structure of the rest of this paper.  
In Section~\ref{sec:algo}, we propose a budget-aware active querying algorithm for two-stage multiple-expert deferral.  
In Section~\ref{sec:generalization}, we introduce the \emph{Sampling-Probs} subroutine used in our algorithm.  
This choice leads to favorable guarantees on the generalization bound of our algorithm.  
In Section~\ref{sec:label}, we establish a label complexity bound for our budgeted two-stage deferral algorithm.  
In the realizable case, the expected number of expert cost queries it issues scales as $\wt O(\sqrt{T})$, a substantial improvement over the linear label complexity $\num T$ incurred by standard two-stage methods that query all expert costs. For simplicity of exposition, the main body presents this square-root bound, while the stronger logarithmic bound (derived via Freedman’s inequality) is given in full detail in Appendix~\ref{app:enhanced}.  
Even in the agnostic setting, the bound remains favorable when the optimal surrogate loss $\sE^*$ is small.  
In Appendix~\ref{app:single-stage}, we also present a novel algorithm for single-stage multiple-expert deferral with budgeted expert queries, supported by analogous theoretical guarantees.  
Finally, in Section~\ref{sec:experiments}, we report experimental results demonstrating the effectiveness of our methods, including comparisons with standard deferral baselines.  
Our budgeted deferral method matches the accuracy of the standard approach while reducing expert queries to below 40\%, with even larger gains as the number of experts increases, demonstrating strong scalability to complex prediction tasks.

\textbf{Novelty.} 
This work introduces a novel two-stage deferral solution that significantly reduces expert cost queries in a cost-sensitive active learning framework.  
Unlike standard active learning, our approach goes beyond deciding whether or not to query a label; instead, for each instance, we determine \emph{which} expert cost to query and with what probability.  
We achieve this by extending existing active learning algorithms and conceptual tools to our setting.  
Our new technical solution offers strong generalization bounds and label complexity guarantees.  
In particular, in the realizable case the label complexity scales with the square-root of the time horizon $T$—and can be further sharpened to logarithmic dependence using Freedman’s inequality (see Appendix~\ref{app:enhanced}).  
We also present a novel solution for the single-stage deferral setting, which incorporates an additional “no deferral” option.  
Our experiments across various binary and multi-class datasets demonstrate significant savings in expert queries (at most one per sample) while maintaining prediction accuracy comparable to full batch settings.

\textbf{Related Work.} 
The most relevant prior work to our study is by \citet{reid2024online},
who model deferral to a single expert as a two-armed contextual bandit
problem with budget constraints. This formulation enables the direct
application of existing bandit algorithms with knapsack constraints
\citep{agrawal2016linear, filippi2010parametric, li2017provably}.
However, extending these methods to multiple experts is non-trivial, and
general bandit algorithms \citep{lattimore2020bandit} are not tailored
to the deferral loss or its consistent surrogate losses, which are
central to learning-to-defer approaches. Moreover, \citet{reid2024online}
assume a generalized linear model for expert performance, an assumption
that is often too restrictive in practice. While multi-armed bandit
(MAB) algorithms are powerful tools, our problem does not naturally fit
this framework. A direct mapping of experts to arms yields regret
benchmarks (e.g., external or shifting regret) that are either
uninformative or misaligned with the objectives of deferral. Even refined
notions such as contextual or policy regret face challenges, as the
reward definition must balance the immediate cost of querying an expert
with the long-term benefit of acquiring labels for training. This trade-off,
together with the dependence of generalization on the choice of policy
class, makes it difficult to cast our setting as a standard bandit
problem. For a detailed discussion, see
Appendix~\ref{app:bandits}.

\section{Two-stage expert deferral framework}
\label{sec:pre}

We consider a standard multi-class classification setting with input
space $\sX$ and label space $\sY = [n] \coloneqq \{1, \ldots, n\}$ for
$n \geq 2$ classes.

In the \emph{two-stage multiple-expert deferral framework}, learning
proceeds in two phases. In Stage 1, a multi-class classification
predictor is trained and then fixed, becoming one of several available
experts.  In Stage 2, a \emph{routing function} is learned, which
selects, on a per-input basis, one of $\num \geq 2$ predefined experts
$\expert_1, \ldots, \expert_{\num}$ (including the trained
predictor). Each expert is a scoring function $\expert_j\colon \sX
\times \sY \to \Rset$, and its prediction on input $x$ is
$\expertexpert_j(x) = \argmax_{y \in \sY} \expert_j(x, y)$.  The
routing function $r \in \sR$ selects an expert index according to:
\[
\rr(x) = \argmax_{k \in [\num]} r(x, k),
\]
with ties broken deterministically. The set $\sR$ is a finite
hypothesis class of scoring functions $r\colon \sX \times [\num] \to
\Rset$, and its cardinality is denoted by $\abs*{\sR}$. Our results
naturally extend to infinite $\sR$ using covering numbers (see Appendix~\ref{app:infinite}).

The \emph{two-stage deferral loss} is defined as:
\[
\tdef(r, x, y, \bc) = \sum_{k = 1}^{\num} c_k(x, y)  \1_{\rr(x) = k},
\]
where $c_k(x, y)$ is the cost incurred for selecting expert
$\expert_k$ and $\bc = \paren*{c_1, \ldots, c_{\num}}$ is the cost vector. A common choice is the $0$-$1$ loss: $c_k(x, y) =
\1_{\expertexpert_k(x) \neq y}$, though $c_k$ may also incorporate
computation or fairness costs, for example $c_k(x, y) = \alpha_k
\1_{\expertexpert_k(x) \neq y} + \beta_k$, for some $\alpha_k, \beta_k
> 0$. We adopt the $0$-$1$ cost formulation for simplicity, following
\citet{MaoMohriMohriZhong2023two}, but our results straightforwardly
extend to other choices.

Since $\tdef$ is non-differentiable and its optimization is
intractable due to the presence of indicator functions in its
definition, we resort instead to a surrogate loss. The surrogate
general loss function proposed by \citet{MaoMohriMohriZhong2023two} is:
\[
\sfL(r, x, y, \bc) = 
\sum_{k = 1}^{\num} \paren*{1 - c_k(x, y)}  \ell(r, x, k),
\]
where $\ell$ is a surrogate loss for multiclass classification. For
example, if $\ell$ is the multiclass logistic loss \citep{Verhulst1838,Verhulst1845,Berkson1944,Berkson1951}, then:
\[
\ell(r, x, k) = \log \paren*{1 + \sum_{k' \neq k} e^{r(x, k') - r(x, k)}}.
\]
We assume (possibly after normalization) that $\sfL$ takes values in
$[0, 1]$.  Let $\sD$ be a distribution over $\sX \times \sY \times
\{0,1\}^{\num}$. The \emph{expected surrogate loss} or
\emph{generalization error} of a routing function $r$ is defined as:
$\sE(r) = \E_{(x, y, \bc) \sim \sD}[\sfL(r, x, y, \bc)]$, and
the \emph{best-in-class generalization error} over $\sR$ is:
$\sE^*(\sR) = \inf_{r \in \sR} \sE(r)$.

\section{Budgeted deferral algorithm} 
\label{sec:algo}

How can we design a two-stage deferral solution that reduces expert
queries? Our approach builds on an existing active learning algorithm
(specifically, IWAL \citep{beygelzimer2009importance}, but other similar algorithms could also be
adapted). However, unlike standard active learning, our problem
requires more than simply deciding whether or not to query a
label. Instead, for each instance $(x_t, y_t)$, we must determine which expert costs $c_k(x_t, y_t)$ to query.

To address this, we decompose the decision into two parts: (1)
selecting an expert $k$, and (2) determining the probability of
querying $c_k(x_t, y_t)$ once $k$ is chosen. As we show in our analysis, selecting experts uniformly gives the best worst-case bounds since all experts are treated equally.\ignore{the optimal choice of expert turns out to be
uniform at random with respect to the worst-case upper bounds since all experts are treated symmetrically.} For the query probability $p_{t,k}$, we carefully
design its expression based on the surrogate loss scores for expert
$k$ computed by each routing function in the current version space,
along with the maximum cost in next section.

This formulation allows us to derive strong theoretical guarantees on
the label complexity. Moreover, our experimental results empirically
validate the effectiveness of this approach, demonstrating significant
savings in expert queries.  Remarkably, our algorithm queries at most
one expert cost per sample.

\textbf{Novelty.} 
What sets our method apart is this novel adaptation
of active learning tools to a deferral-based cost-sensitive
setting. Unlike prior work that typically focuses on binary query
decisions or uniform cost assumptions, we derive a principled
two-stage mechanism that incorporates both expert selection and
cost-sensitive querying into a single coherent framework.

Algorithm~\ref{alg:IWAL} outlines the core procedure of our budgeted
two-stage deferral strategy in the multi-expert setting.
\begin{algorithm}[h]
  \caption{Budgeted Two-Stage Deferral with Multiple Experts
    (Subroutine $\textsc{Sampling-Probs}$)}
\label{alg:IWAL}
\begin{algorithmic}
\STATE \textsc{Initialize} $S_0 \gets \emptyset$;
\FOR{$t = 1$ \TO $T$}
\STATE \textsc{Receive}$(x_t, y_t)$;
\STATE $p_t \gets \textsc{Sampling-Probs}\newline
(x_t, y_t, \curl*{x_s, y_s, \bc_s, q_s, k_s, p_s, Q_s
\colon 1 \leq s < t})$;
\STATE $k_t \gets \textsc{Sample}(\num, q_t)$;
\STATE $Q_{t, k_t} \gets \textsc{Bernoulli}(p_{t, k_t} )$;
\IF{$Q_{t, k_t} = 1$}
\STATE $c_{t, k_t} \gets \textsc{Query-Cost}(k_t, (x_t, y_t))$
\STATE $S_{t} \gets S_{t-1} \cup \curl*{\paren*{x_t, y_t, c_{t, k_t},
    \frac{1}{q_{t, k_t} p_{t, k_t}}}}$;
\ELSE
\STATE $S_t \gets S_{t - 1}$;
\ENDIF
\STATE $r_t \gets \argmin_{r \in \sR} \sum_{(x, y, c, w) \in S_{t}}
w (1 - c)  \ell(r, x, k_t)$.
\ENDFOR
\end{algorithmic}
\end{algorithm}
For any $t \in [T]$ and $k \in [\num]$, we denote by $q_{t, k}$ the
probability of selecting expert $k$ at time $t$. The optimal choice
for the value of $q_t = (q_{t,1}, \ldots, q_{t, \num})$ is determined
in Section~\ref{sec:label}, using our theoretical bounds.

At each round $t$, upon observing a labeled instance $(x_t, y_t)$, the
learner calls a subroutine \emph{Sampling-Probs} (to be detailed
later) that takes as input the current instance and historical data
and returns the vector $p_t = (p_{t, 1}, \ldots, p_{t, \num})$, where
$p_{t, k}$ is the probability of querying the cost $c_{t, k}$
associated with expert $k$, conditioned on expert $k$ being selected.

An expert $k_t$ is then selected according to $q_t$, and its cost is
queried with probability $Q_{t, k_t}  \sim 
\textsc{Bernoulli}(p_{t,k})$.
The algorithm incrementally builds a labeled dataset, assigning an
importance weight to each queried cost. Specifically, if expert $k$ is
selected and its cost $c_{t, k}$ is queried, the example is stored
with a weight of $1/(q_{t,k}  p_{t,k})$ to account for the sampling
process.

Let $\sD$ denote a distribution over $\sX \times \sY \times \curl*{0,
  1}^{\num}$. Then, the expected surrogate loss of a hypothesis $r \in \sR$
is:
\begin{align*}
\sE(r) &= \E_{(x, y, \bc) \sim \sD}[\sfL(r, x, y, \bc)] =  \E_{(x, y, \bc) \sim \sD} \bracket*{\sum_{k = 1}^{\num} \paren*{1 - c_k(x, y)}  \ell(r, x, k)}.
\end{align*}
\ignore{
  Since the distribution $\sD$ and surrogate loss $\sur$ are fixed and understood from context, we omit them in the notation going forward.
  }

To estimate $\sE(r)$ from data, the algorithm uses an importance
weighted empirical estimate at time $T$:
\begin{equation*}
\sE_{T}(r) = \frac{1}{T} \sum_{t = 1}^{T} \sum_{k = 1}^{\num}
\frac{1_{k = k_t}  Q_{t, k}}{q_{t, k}  p_{t, k}} (1 - c_{t, k}(x_t,
y_t))  \ell(r, x_t, k),
\end{equation*}
where $(k_t, Q_{t, k})$ are the random decisions used in sampling and
querying at round $t$. It is straightforward to verify that this
estimator is unbiased, i.e., $\E[ \sE_{T}(r)] = \sE(r)$, where the
expectation is over the internal randomness of the
algorithm. Theorem~\ref{thm:loss-bound} establishes high-probability
concentration bounds for $\sE_T(r)$, assuming the probabilities $p_{t,
  k}$ are chosen appropriately.

\section{Sampling-Probs Strategy and Generalization Guarantees}
\label{sec:generalization}

A key component of the subroutine introduced in this section is the
specific definition of the probabilities $p_{t, k}$, which depend on
the selected expert $k$. We will demonstrate that this choice leads to
favorable guarantees on the generalization bound and label complexity of our algorithm.

Algorithm~\ref{alg:threshold} gives the pseudocode of the
\emph{Sampling-Probs} subroutine used within the budgeted two-stage
multiple-expert deferral framework. This subroutine maintains a
dynamically evolving subset of the hypothesis set (version space),
denoted by $\sR_t$, which is refined over time based on empirical
performance.

\begin{algorithm}[h]
\caption{Sampling-Probs Subroutine with Past History}
\label{alg:threshold}
\begin{algorithmic}
\STATE \textsc{Initialize} $\sR_0 \gets \sR$;
\FOR{$t = 2$ \TO $T$}
\STATE $\sE_{t - 1}(r) \gets \frac{1}{t - 1}  \sum_{s = 1}^{t - 1}
\sum_{k = 1}^{\num} \frac{1_{k = k_s} Q_{s, k} }{q_{s, k} p_{s, k}}\newline
 \times (1 - c_{s, k}(x_s, y_s))  \ell(r, x_s, k)$
\STATE $\sE_{t - 1}^{*}\gets \min_{r \in \sR_{t - 1}}\sE_{t - 1}(r) $;
\STATE $\sR_t \gets \curl*{r \in \sR_{t - 1}\colon
\sE_{t - 1}(r)
\leq \sE_{t - 1}^{*} + \Delta_{t - 1} }$;
\STATE $p_{t, k} \gets \max_{r, r' \in \sR_t} \curl*{\ell(r, x_t, k) - \ell(r', x_t, k)}$.
\ENDFOR
\end{algorithmic}
\end{algorithm}

The version space is initialized with the full hypothesis class
$\sR$. After each round, it is pruned to retain only those hypotheses
whose empirical error does not exceed that of the current best
predictor (within $\sR_t$) by more than a slack parameter
$\Delta_t$. Formally:
\begin{equation}
\sR_{t + 1} = \{r \in \sR_t\colon \sE_t(r) \leq \sE_t^* + \Delta_t \},  
\end{equation}
where $\sE_t^* = \min_{r \in \sR_t} \sE_t(r)$ denotes the minimal
empirical error at round $t$.

To control the size of $\Delta_t$, we define $q_{\min} = \min_{k \in
  [\num]} q_{t, k}$, which we assume is strictly positive without loss
of generality, and let $\ov q = \frac{1}{q_{\min}} + 1$. Then, the
threshold $\Delta_t$ is chosen as follows, based on standard sample
complexity arguments:
\begin{equation}
  \Delta_t = \sqrt{\ov q^2 \cdot 8 / t \cdot \log
    \paren*{2t (t + 1)\abs*{\sR}^{2} / \delta}}  .
\end{equation}
This pruning strategy guarantees, with high probability, that the
optimal predictor $r^* \in \sR$ remains in the version space $\sR_t$
at all times with high probability, while progressively eliminating
suboptimal hypotheses (see Theorem~\ref{thm:loss-bound}).

For each instance $(x_t, y_t)$, the subroutine evaluates the
informativeness of querying each expert’s cost by examining the
variability in the expert-specific component of the surrogate loss
across hypotheses in $\sR_t$, leveraging the decomposability of the
loss. The sampling probability $p_{t,k}$ for expert $k$ is then set to
the maximum difference in this component over all pairs of hypotheses
in $\sR_t$:
\begin{align*}
p_{t, k} 
&= \max_{r, r' \in \sR_t} \max_{c \in \curl*{0, 1}} (1 - c) \paren*{\ell(r, x_t, k) - \ell(r', x_t, k)}\\
&= \max_{r, r' \in \sR_t} \curl*{\ell(r, x_t, k) - \ell(r', x_t, k)}.    
\end{align*}
Since the surrogate loss is normalized within $[0,1]$, the resulting
sampling probabilities are also bounded in this range.

This design allocates the query budget adaptively, prioritizing
experts and instances where the disagreement among remaining
hypotheses is greatest, thus targeting high-uncertainty regions. 

We now establish high-probability performance guarantees for the
predictors output by the algorithm.

\begin{restatable}[\textbf{Two-Stage Generalization Bound}]
  {theorem}{GeneralizationBound}
\label{thm:loss-bound} 

Let $\sD$ be any distribution over $\sX \times \sY \times \curl*{0,
1}^{\num}$, and let $\sR$ be a hypothesis class. Assume that $r^* \in
\sR$ minimizes the expected surrogate loss $\sE(r)$. Then, for any
$\delta > 0$, with probability at least $1 - \delta$, the following
holds for all $T \geq 1$:
\begin{itemize}
\item The optimal hypothesis $r^*$ belongs to the retained set
  $\sR_T$;
\item For all $r, r' \in \sR_T$, the generalization gap satisfies
\[
\sE(r) - \sE(r') \leq 2 \Delta_{T - 1}.
\]
\end{itemize}
In particular, the learned hypothesis $r_T$ at time $T$ satisfies
\[
\sE(r_T) - \sE(r^*) \leq 2 \Delta_{T - 1}.
\]
\end{restatable}
The proof is given in Appendix~\ref{app:loss-bound}. It relies on a
concentration argument built around Lemma~\ref{lemma:pair}, which is
stated and proved in Appendix~\ref{app:pair}. 

\section{Label Complexity}
\label{sec:label}

In the previous section, we established that the generalization error
of our budgeted two-stage deferral algorithm closely matches that of a
standard deferral method with full access to all $\num T$ expert
costs. We now turn to analyzing the label complexity of our
approach, that is, the expected number of expert cost queries it
issues.

To derive label complexity guarantees for our algorithm,
we must adapt existing tools and definition in active
learning to our deferral setting. In particular, we
will define a new notion of slope asymmetry, hypothesis
distance metric, generalized disagreement
coefficient, based on experts' costs $c_k$ and
tailored to our setting.
These tools allow us to demonstrate that our budgeted
deferral algorithm can achieve a favorable label complexity, in fact
lower than its fully supervised counterpart when the learning problem
is approximately realizable and the disagreement coefficient of the
hypothesis set is not loo large.

We focus on a family of multiclass surrogate losses $\ell$ that
includes, among others, the multinomial logistic loss. We also assume
a mild condition on the cost structure: for every input-label pair,
there exists at least one expert that incurs zero cost. This
assumption can be satisfied by expanding the expert pool to include a sufficiently diverse set of experts, such that at least one performs well on each instance.\ignore{This
assumption holds in many practical deferral settings.}
As in \citep{beygelzimer2009importance}, a central property we require
of the loss function $\ell$ is bounded slope asymmetry, which controls
how differences in surrogate losses between hypotheses may be
distorted by cost vectors. This condition is key to relating
loss-based disagreement to label complexity.

\begin{definition}[Slope Asymmetry for Two-Stage Deferral]
\label{def:slope-asymmetry}
The \emph{slope asymmetry} of a multi-class loss function $\ell: \sR
\times \sX \times [\num] \to [0,\infty)$ is defined as: $K_{\ell} =$
\begin{align*}
\sup_{r, r', x, y} 
\frac{\max_{\bc \in \curl*{0,1}^{\num}} \sum_{k = 1}^{\num} \paren*{1 - c_k}
  \abs*{\ell(r, x, k) - \ell(r', x, k)}}{\min_{\bc \in \curl*{0,1}^{\num}} \sum_{k = 1}^{\num} \paren*{1 - c_k} \abs*{\ell(r, x, k) - \ell(r', x, k)}}.  
\end{align*}
\end{definition}
\vspace{-1.5pt}
This quantity is always well-defined and finite if, for every $(x, y)$,
there exists at least one expert $k^*$ with zero cost, $c_{k^*}(x, y) = 0$.
In practice, $K_{\ell}$ is bounded for common convex surrogates such as the logistic loss, provided the range of score functions $r$ is
restricted to a compact interval (e.g., $[-B, B]$); see
Appendix~\ref{app:slope-asymmetry} for details and explicit bounds. Next, we define a distance measure over the
hypothesis set that reflects variability in expert-specific loss
components.

\begin{definition}[Hypothesis Distance Metric]
For any $r, r' \in \sR$ and distribution $\sD$, define $\rho(r, r') =$
\begin{align*}
\E_{(x, y) \sim D} \bracket*{ \max_{\bc \in \curl*{0,1}^{\num}}
    \sum_{k = 1}^{\num} \paren*{1 - c_k} \abs*{\ell(r, x, k) - \ell(r', x, k)}}.
\end{align*}
Define $\e$-ball around $r$ as $B(r, \e) = \curl[\big]{r' \in \sR \colon \rho(r, r') \leq \e }$.
\end{definition}
Suppose $r^* \in \sR$ minimizes the expected surrogate loss: $\sE^* =
\sE(r^*) = \inf_{r \in \sR} \sE(r)$. At time $t$, the version space
$\sR_t$ contains only hypotheses with generalization error at most
$\sE^* + 2\Delta_{t-1}$. But how close are these hypotheses to $r^*$
in $\rho$-distance? The following lemma provides an upper bound in
terms of the slope asymmetry:
\begin{restatable}{lemma}{TwoSpaces}
\label{lemma:two-spaces}
For any distribution $\sD$ and any multi-class loss function $\ell$,
we have $\rho(r, r^*) \leq K_{\ell} \cdot ( \sE(r) + \sE^* )$ for all
$r \in \sR$.
\end{restatable}
The proof is provided in Appendix~\ref{app:two-spaces}.
The following extends the notion of disagreement \citep{hanneke2007bound} to our setting:
\begin{definition}
The \emph{disagreement coefficient} $\theta$ is the smallest value
such that, for all $\e > 0$,
\begin{equation*}
  \E_{(x, y) \sim D} \sup_{r \in B(r^*, \e)}
  \sup_{k \in [\num]} \abs*{\ell(r, x, k) - \ell(r^*, x, k)} \leq  \theta \e.   
\end{equation*}
\end{definition}
We now present an upper bound on the expected number of cost queries
required by the algorithm. The proof is included in
Appendix~\ref{app:label}.
\begin{restatable}[\textbf{Two-Stage Label Complexity Bound}]{theorem}{LabelBound}
\label{thm:label} 
 Let $\sD$ be a two-stage deferral distribution and $\sR$ a hypothesis
set. Suppose the loss function $\ell$ has slope asymmetry $K_{\ell}$
and the disagreement coefficient of the problem is $\theta$. Then,
with probability at least $1 - \delta$, the expected number of cost
queries made by the budgeted two-stage deferral algorithm over $T$
rounds is bounded by.
\begin{equation}
4\theta \cdot K_{\ell} \cdot 
\paren*{\sE^* T + O\paren*{ \paren*{1 / q_{\min} + 1} \sqrt{T \log (\abs*{\sR} T/\delta)} } },    
\end{equation}
where the expectation is taken over the algorithm's randomness.
\end{restatable}
The theorem establishes a label complexity bound for our budgeted
two-stage deferral algorithm. In the realizable case, the bound scales
as $\wt O(\sqrt{T})$, significantly improving over the linear label
complexity $\num T$ incurred by standard two-stage methods that query
all expert costs. For simplicity of exposition, the main body presents this square-root bound, while the stronger logarithmic bound (derived via Freedman’s inequality) is given in Theorem~\ref{Thm:bdef} in Appendix~\ref{app:enhanced}.  Even in the agnostic setting, the bound remains
favorable when the optimal surrogate loss $\sE^*$ is small. A high-probability version of this result is provided in 
Corollary~\ref{cor:hp-label} in Appendix~\ref{app:high-probability}.

The dependence on the generalized disagreement coefficient is, in
general, unavoidable, as shown by \citet{Hanneke2014}. However, this
coefficient has been shown to be bounded for many common hypothesis
classes, enabling meaningful guarantees in practice.

\paragraph{Optimal $q_t$} Finally, we note that both the generalization
and label complexity bounds are minimized when the expert sampling
probabilities are uniform: $q_{t,k} = 1/\num$ for all $t$ and $k$,
yielding $q_{\min} = 1/\num$ since all experts are treated symmetrically. Under this setting, the bounds simplify
to:
\begin{align}
\label{eq:bounds-linear}
\sE(r_T) \leq \sE(r^*) + 2(\num + 1) \sqrt{(8 / (T - 1)) \log (2(T - 1) T \abs*{\sR}^{2} / \delta)} \nonumber \\
\E \bracket*{\sum_{t = 1}^T \sum_{k = 1}^{\num} 1_{k_t = k} Q_{t, k}} \leq 4\theta \cdot K_{\ell} \cdot 
\paren*{\sE^* T + O\paren*{ \paren*{\num + 1} \sqrt{T \log (\abs*{\sR} T/\delta)} } }. 
\end{align}
\ignore{
These bounds highlight the efficiency of the proposed approach, which
significantly reduces label usage compared to standard deferral
methods that query all $\num$ experts at every round, resulting in a
total of $\num T$ queries.
}
The linear dependence on $\num$ also appears in standard deferral methods. For example,  \citet[Theorem~3]{MaoMohriZhong2024deferral} establishes a generalization bound with linear dependence on the number of experts. 
By leveraging Freedman’s
inequality \citep{freedman1975tail} in place of Azuma’s \citep[Theorem~D.7]{MohriRostamizadehTalwalkar2018} in our analysis, we can in fact derive
learning and sample complexity bounds that depend only logarithmically
on $T$ in the realizable case (see Appendix~\ref{app:enhanced}). This significantly strengthens our
theoretical guarantees and further highlights the advantages of our
approach.

\begin{figure*}[t]
    \centering
    \includegraphics[scale=.3]{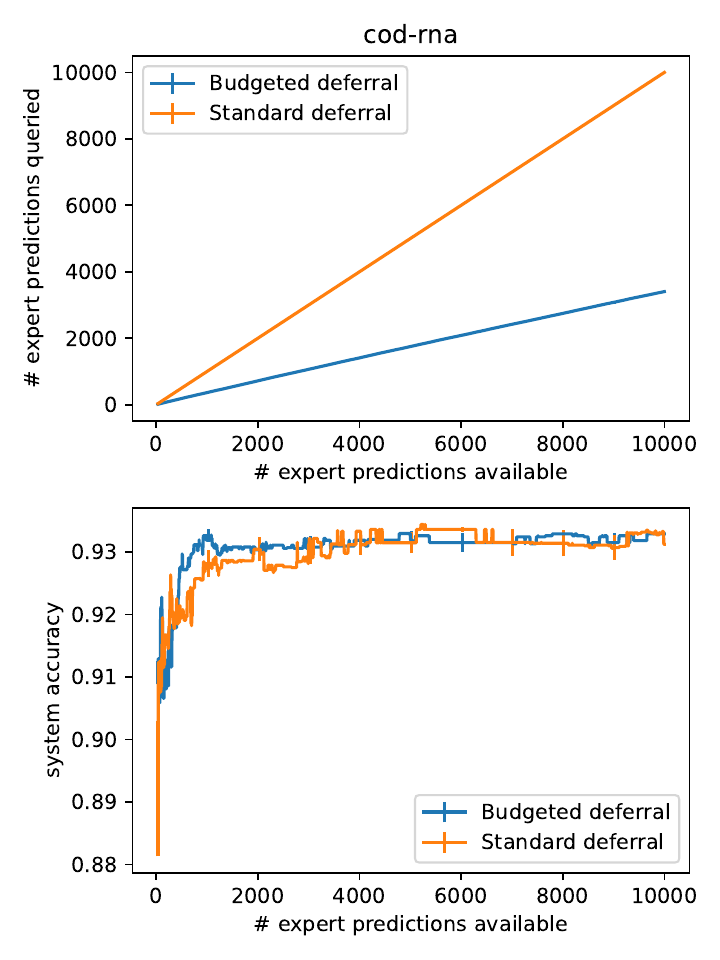}
    \includegraphics[scale=.3]{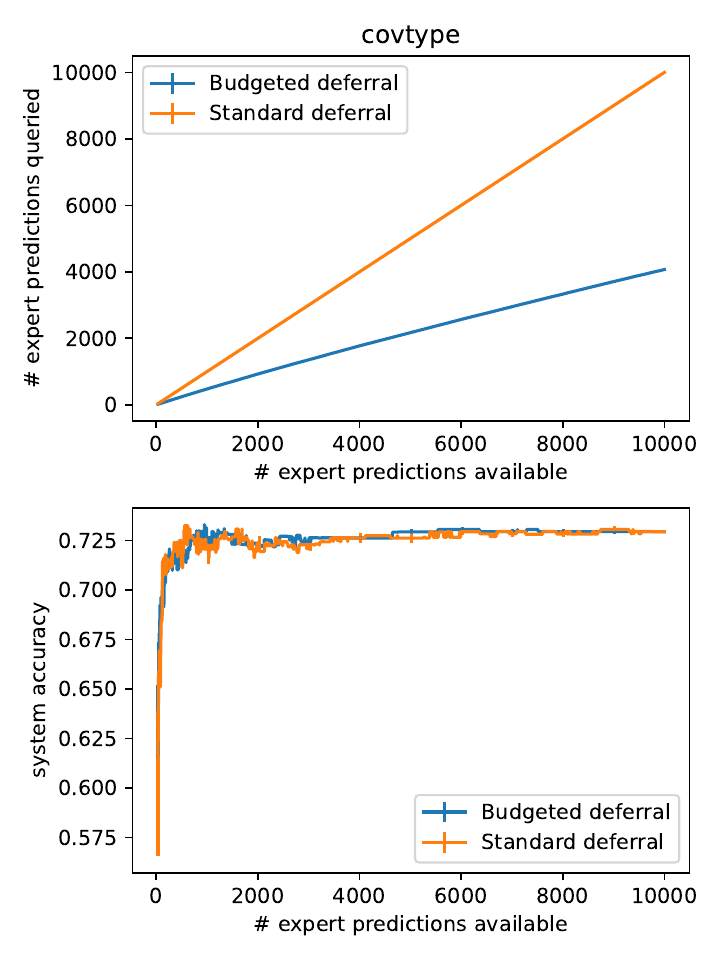}
    \includegraphics[scale=.3]{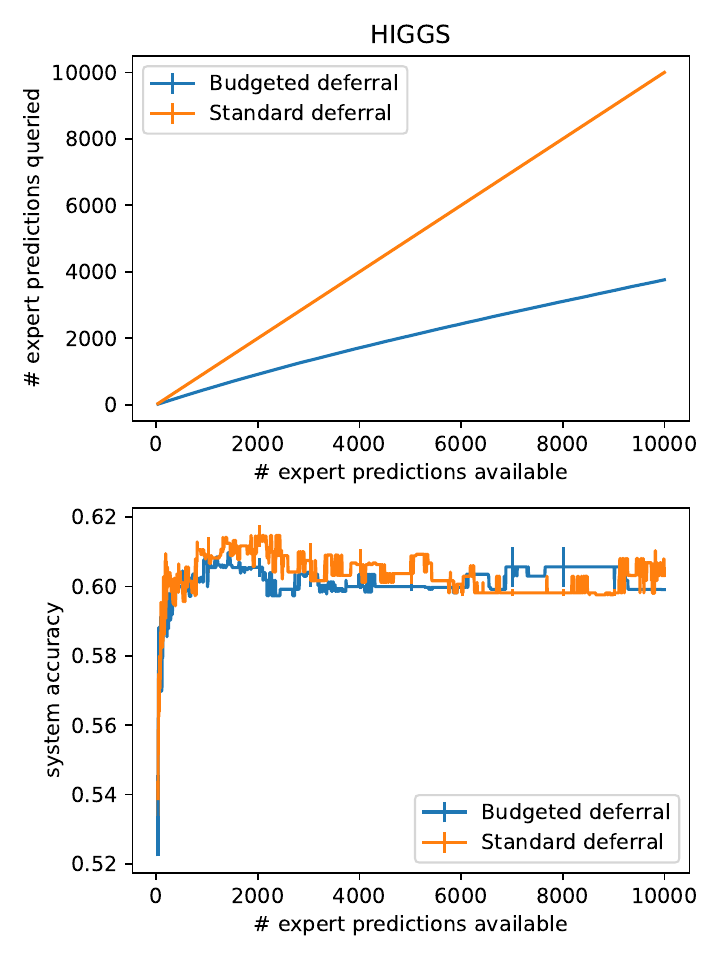}\\
    \includegraphics[scale=.3]{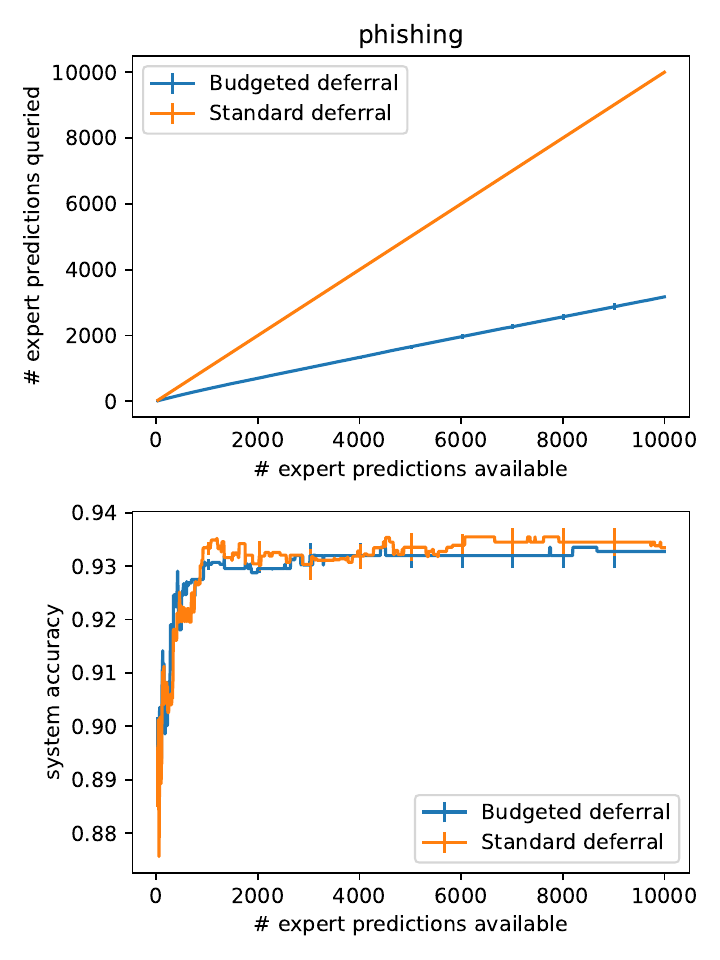}
    \includegraphics[scale=.3]{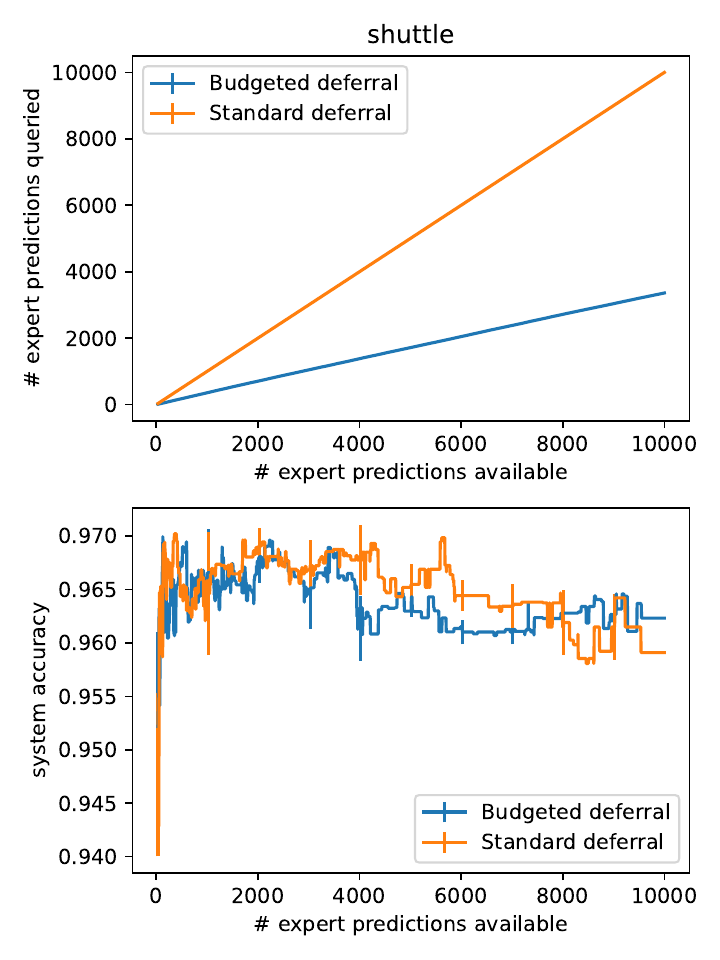}
    \includegraphics[scale=.3]{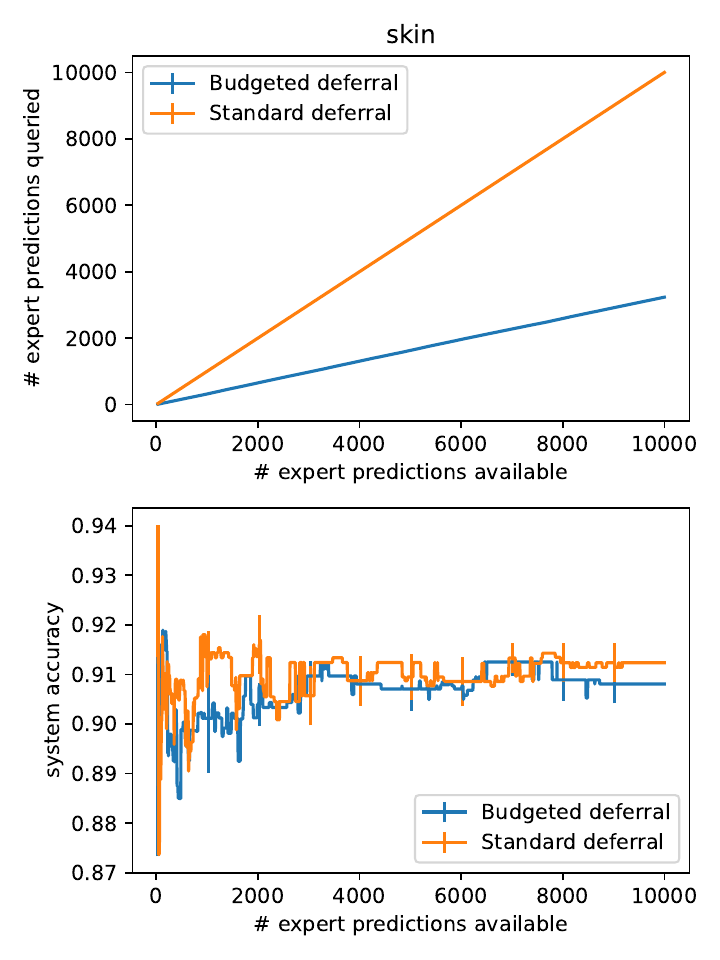}
    \caption{Standard vs. Budgeted Two-Stage Multiple-Expert Deferral on Binary Datasets.}
    \label{fig:tdef}
\end{figure*}

\begin{figure*}[t]
    \centering
    \includegraphics[scale=.25]{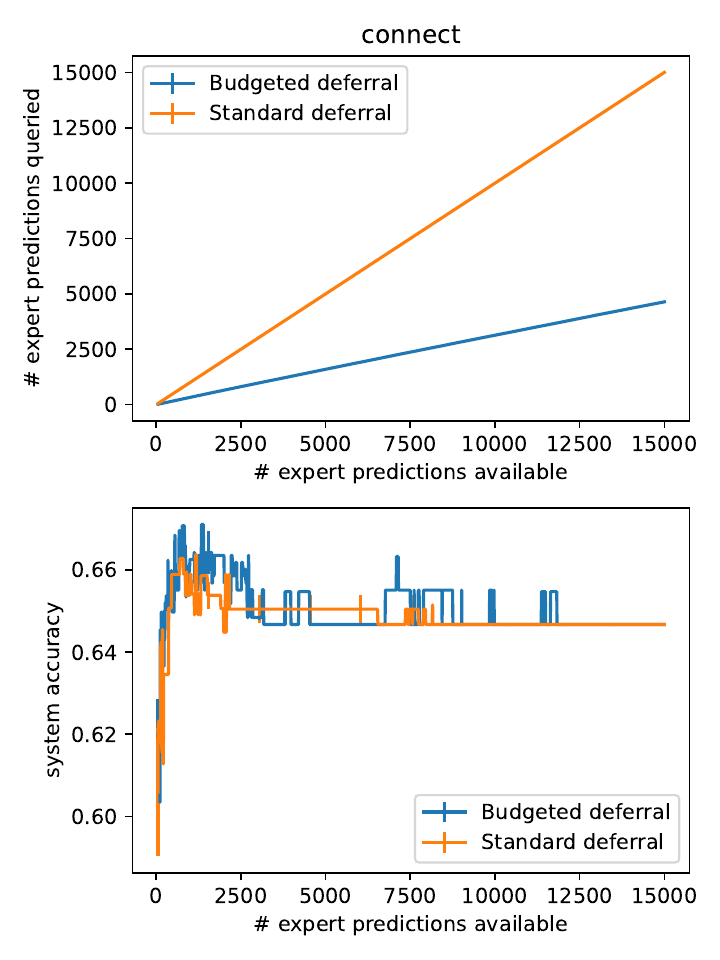}
    \includegraphics[scale=.25]
    {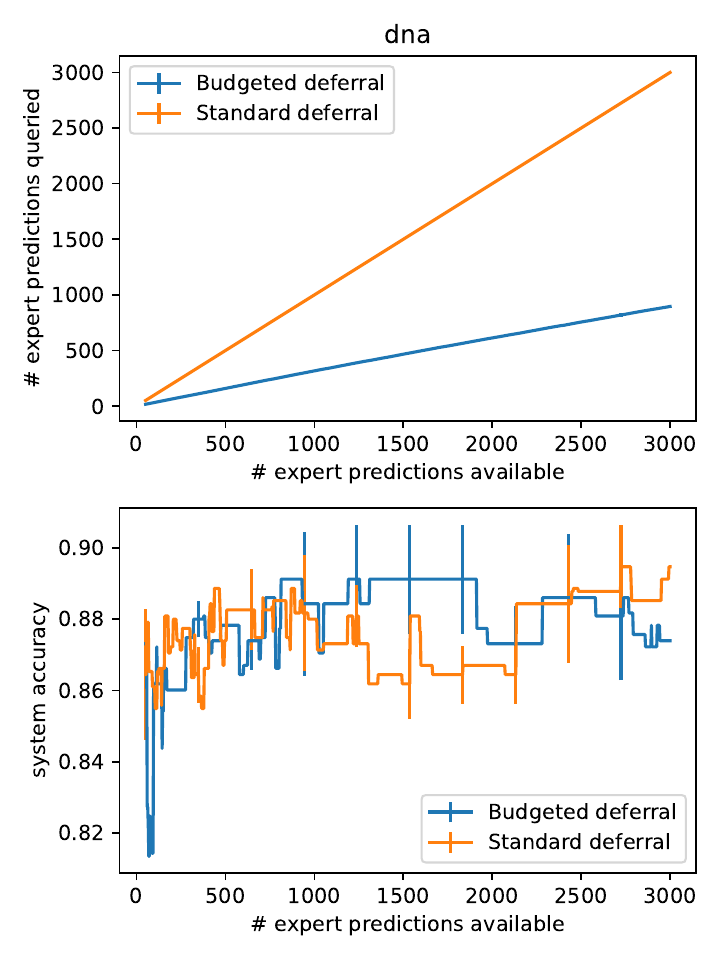}    
    \includegraphics[scale=.25]{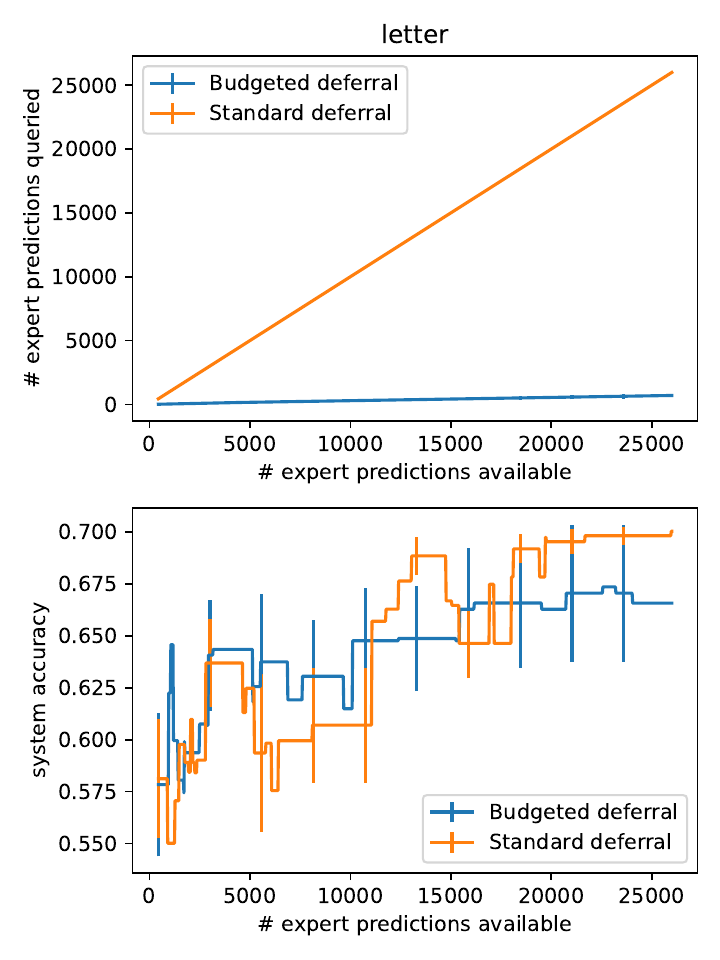}
    \includegraphics[scale=.25]{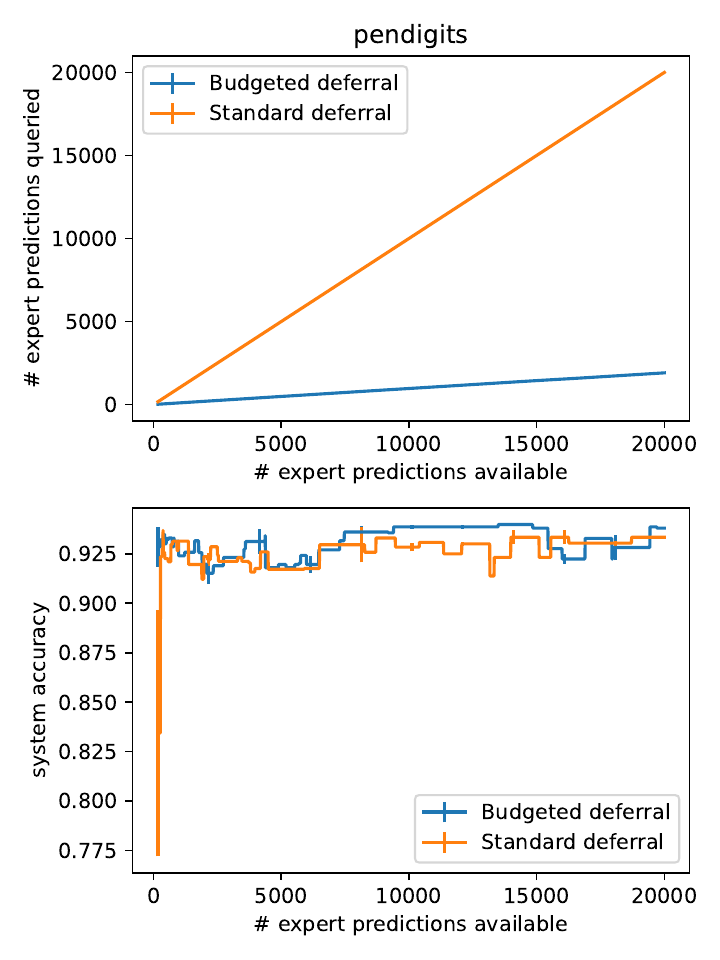}
    \caption{Standard vs. Budgeted Two-Stage Multiple-Expert Deferral on Multi-Class Datasets.}
    \vskip -.1in
    \label{fig:tdef-multi}
\end{figure*}

\section{Practical Implementation}
\label{sec:easy-alg}

In this section, we show that the budgeted two-stage deferral algorithm
can be efficiently implemented for convex \emph{comp-sum losses}
\citep{mao2023cross} and hypothesis classes $\sR$ defined as convex
sets of score functions $r \colon \sX \times [n] \to \Rset$.
The results apply not only to linear predictors but also to a broader
range of models, including neural networks (where convexity is lost but
the same optimization routines can be approximated via SGD).

We consider \emph{comp-sum losses} \citep{mao2023cross} as the
multi-class surrogate loss family $\ell$, which includes many popular
losses such as the multinomial logistic loss. A comp-sum loss is defined
for any $(r, x, k) \in \sR \times \sX \times [n]$ as
\begin{equation*}
\ell_{\mathrm{comp}}(r, x, k)
= \Psi \paren*{\frac{e^{r(x, k)}}{\sum_{k' \in [n]} e^{r(x, k')}}},
\end{equation*}
where $\Psi \colon [0, 1] \to \Rset_{+} \cup \curl{+\infty}$ is a
non-increasing function. A notable instance is $\Psi(u) = - \log u$,
which yields the \emph{logistic loss}
\citep{Verhulst1838,Verhulst1845,Berkson1944,Berkson1951}.
For suitable choices of $\Psi$, the loss $\ell_{\mathrm{comp}}$
is convex in $r$.

\paragraph{Convex feasible region.}
At each round $t$, Algorithm~\ref{alg:threshold} requires solving two
optimization problems over the restricted hypothesis set $\sR_t$, defined
as the intersection of convex constraints accumulated up to round $t$: $\sR_t = \bigcap_{t' < t} \curl[\Big]{r \in \sR \colon \frac{1}{t'} 
\sum_{i = 1}^{t'} \sum_{k = 1}^{\num} \frac{1_{k_i = k} Q_{i, k} }{q_{i, k} p_{i, k}} (1 - c_{i, k}(x_i, y_i)) \ell_{\rm{comp}}(r, x_i, k)
\leq \sE_{t'}^{*} + \Delta_{t'} }.    $
Since $\ell_{\mathrm{comp}}$ is convex in $r$, each constraint defines
a convex set, and thus $\sR_t$ is convex.

\paragraph{First optimization.}
The first optimization problem at each round $t$ computes the minimal empirical loss:
$
\sE_{t}^{*}
= \min_{r \in \sR_{t}}
\frac{1}{t} \sum_{i = 1}^{t} \sum_{k = 1}^{n}
\frac{1_{k_i = k} Q_{i, k}}{q_{i, k} p_{i, k}} (1 - c_{i, k}(x_i, y_i)) \ell_{\mathrm{comp}}(r, x_i, k).
$
This is a convex program in $r$ over the feasible region $\sR_t$.

\paragraph{Second optimization.}
The second optimization determines the sampling probability $p_{t, k}$
by maximizing the difference in surrogate losses for each expert $k$:
$
\max_{r, r' \in \sR_t}
\Bigl\{
\Psi \paren*{\frac{e^{r(x, k)}}{\sum_{k' \in [n]} e^{r(x, k')}}}
- \Psi \paren*{\frac{e^{r'(x, k)}}{\sum_{k' \in [n]} e^{r'(x, k')}}}
\Bigr\}.
$
Since $\Psi$ is non-increasing, this value is maximized when one term is
minimized and the other maximized. Define
$
S_{\min}(x, k) \equiv \min_{r \in \sR_t}
\frac{e^{r(x, k)}}{\sum_{k' \in [n]} e^{r(x, k')}}$ and $
S_{\max}(x, k)
\equiv \max_{r \in \sR_t}
\frac{e^{r(x, k)}}{\sum_{k' \in [n]} e^{r(x, k')}}.
$
Then the optimal loss variation equals
$\Psi\paren*{S_{\min}(x, k)} - \Psi\paren*{S_{\max}(x, k)}$.

\paragraph{Interpretation.}
This shows that for convex comp-sum losses and convex hypothesis sets
$\sR$, both optimization problems required by
Algorithm~\ref{alg:threshold} are convex and can be solved efficiently.
As an example, when $\sR$ is the linear class
$r(x, k) = \bw \cdot \Phi(x, k)$ with $\norm{\bw} \leq B$,
the optimizations reduce to convex problems in $\bw$.
For neural networks, the problems are no longer convex but can be
handled in practice with standard stochastic gradient descent (SGD).

\paragraph{Heuristics.}
As in IWAL \citep{beygelzimer2009importance}, practical heuristics can be used to
further simplify the implementation. For the first optimization, one can
minimize over $\sR$ instead of $\sR_t$, and for the second optimization,
it suffices to impose only the most recent constraint (from round $t-1$)
instead of all past constraints. With these simplifications, the optimal
solution remains within the feasible set, while computational efficiency
is improved.

\section{Experiments}
\label{sec:experiments}

This section evaluates the empirical performance of our proposed
algorithm against an existing baseline in two-stage, multiple-expert
learning-to-defer scenarios. We aim to confirm that our algorithm
achieves two key outcomes predicted by our theoretical analysis:
maintaining performance guarantees for deferral (i.e., minimizing
deferral loss) and controlling computational cost, as measured by the
number of queried expert predictions.

\paragraph{Experimental Setup}  For the budgeted two-stage
multiple-expert deferral framework, we tested our approach,
Algorithm~\ref{alg:IWAL}, using ten publicly available benchmark
datasets:  \texttt{shuttle}, \texttt{cod-rna}, \texttt{covtype},
\texttt{HIGGS}, \texttt{skin}, \texttt{phishing},  \texttt{dna},  \texttt{connect-4}, \texttt{letter}, and \texttt{pendigits}\footnote{https://www.csie.ntu.edu.tw/$\sim$cjlin/libsvmtools/datasets/}. We compared our method to the state-of-the-art two-stage multiple-expert deferral
algorithm \citep{MaoMohriMohriZhong2023two}, which serves as the
baseline. This algorithm imposes no budget restriction: it queries and
uses the predictions of all $n_e$ experts at every round, incurring a
total cost of $t \times n_e$ expert predictions by time $t$. In the figures, the orange baseline curves correspond to this
unconstrained state-of-the-art method, which incurs the full cost of
$t \times n_e$ queries, while the blue curves correspond to our
budgeted algorithm that makes at most one expert query per round.

 We chose $\ell$ as the multinomial logistic loss for both our
method and the baseline. For each dataset, the evaluation of the
approaches was conducted over $5$ trials, each with a random split of
the unlabeled pool and test set.  For each dataset, we generated a
finite hypothesis set $\sH$ of logistic regression models for deferral.  For each dataset, features were normalized
(zero mean, unit variance) and then scaled to ensure the maximum
feature vector had unit norm. 

Table~\ref{tb:datasets} reports the number of features, the counts of test data points, the size of the hypothesis set, and the regularization parameter $C$.  For all experiments, the finite hypothesis set consisted of logistic regression models from the scikit-learn library \citep{scikit-learn}. These models were trained using the {\tt``liblinear''} solver for binary dataset and the {\tt ``lbfgs''} solver for multi-class datasets as well as an $L_2$ regularization parameter $C$, whose specific values are detailed in Table~\ref{tb:datasets}. Each hypothesis was trained on a random sample of data, with the sample size uniformly selected from the interval [30, 500]. We considered expert setups with clearly defined expertise across classes. Our approach assigns one expert per class. More precisely, for datasets with $n$ classes, we define $n$ experts, $\expert_1, \ldots, \expert_n$, such that $\expert_k$ is always correct on instances of the $k$-th class class and predicts uniformly at random on instances from any other class.

The cost function is chosen as the $0$–$1$ loss: $c_k(x, y) =
\1_{\expertexpert_k(x) \neq y}$. Following the IWAL-related literature (e.g., \citet{cortes2019rbal,cortes2019disgraph,amin2020understanding}), we use publicly available datasets that are standard in active learning. We further adopt the zero-one cost from \citep{MaoMohriMohriZhong2023two} and expert setup from \citet{MaoMohriZhong2025}, where experts are accurate on specific classes. These choices enable controlled evaluation and ensure alignment with prior work.

\begin{table}[t]
\caption{Dataset Statistics: Number of Features (\# F), Classes (\# C), Test Data (\# T), Model Class Size $|\sH|$, and Regularization Parameter $C$.}
\label{tb:datasets}
\begin{center}
\begin{tabular}{l|llHlll}
DATASET & \# F & \# C & \# ``UNLABELED'' & \# T & $|\sH|$ & $1/C$ \\
\hline
{\tt Cod-rna} & 8 & 2&  25000 & 34535 & 4096 & $2^{-13}$ \\
{\tt Covtype} & 54 &  2& 400000 & 181012 & 4096 & $2^{-13}$ \\ 
{\tt Shuttle} & 9 &  2&20000 & 14500 & 4096 & $2^{-13}$ \\
{\tt Skin} & 3 &  2&150000 & 95057 & 2048 & $2^{-11}$ \\
{\tt Phishing} & 68 &  2& 9000 & 2055 & 2048 & $2^{-11}$ \\
{\tt Dna} & 180 &  3 & 9000 & 1186 & 2048 & $2^{-13}$ \\
{\tt Letter} & 16 & 26 & 9000 & 5000 & 2048 & $2^{-13}$ \\
{\tt Connect} & 126 &  3 & 9000 & 66457 & 2048 & $2^{-13}$ \\
{\tt Pendigits} & 16 &  10 & 9000 & 3498 & 2048 & $2^{-13}$ \\
{\tt HIGGS}
\footnote{As a preprocessing step, a uniform random subsample of the HIGGS data was taken, resulting in 100,000 data points.}
& 28 &  2&  90000 &  10000 & 2048 & $2^{-13}$
\end{tabular}
\end{center}
\vskip -0.3in
\end{table}

\paragraph{Results} Figure~\ref{fig:tdef} and Figure~\ref{fig:tdef-multi} demonstrate our method’s superior performance across ten datasets, achieving comparable system accuracy, defined as the average value of $\bracket*{1 - \tdef(h, x, y)}$ on the test data, to the standard deferral algorithm, while substantially reducing the number of expert queries.

The Expert Predictions Available (X-axis) represents the cumulative number of expert predictions that could have been made up to time step $t$, which equals $t \times n_e$. The Expert Predictions Queried (Y-axis of the top plots) is the cumulative number of times the algorithm actively queried an expert. Following Algorithm~\ref{alg:IWAL}, this corresponds to $\sum_{s=1}^t Q_{s,k_s}$, while in the standard deferral setting we set this equal to $t \times n_e$.

In Figure~\ref{fig:tdef} (binary datasets) and Figure~\ref{fig:tdef-multi} (multi-class datasets), our budgeted deferral method achieves the same accuracy while issuing far fewer expert queries than the standard method (which uses $n_e$ expert predictions at every time step). The top plots show the exact reduction in queries, i.e., $\sum_{s=1}^t Q_{s,k_s} < t \times n_e$, while the bottom plots confirm that accuracy is preserved in both methods.

Because the standard method requires predictions from all $n_e$ experts at every round, it accumulates a total of $t \times n_e$ expert predictions by time $t$. In contrast, our budgeted method makes at most one expert query per round. Consequently, the blue (budgeted) and orange (standard) curves diverge sharply in Figure~\ref{fig:tdef-multi}, with the budgeted method achieving comparable accuracy while reducing expert queries by a large margin. This highlights the strong efficiency gains of our approach, particularly in settings with many experts.

While the exact percentage reduction varies by dataset, for binary datasets (Figure~\ref{fig:tdef}) it generally lies between 35–40\%, already representing a substantial decrease. For multi-class datasets (Figure~\ref{fig:tdef-multi}), the reduction is even larger, consistently below 30\%. Moreover, as the number of experts increases, particularly in the \texttt{letter} dataset with 26 experts and the \texttt{pendigits} dataset with 10 experts, the performance improvements of our method become much more significant, underscoring its strong scalability to complex prediction tasks. Interestingly, in a previous version of this work, we stated that the multi-class results were less favorable than in the binary case; however, after correcting the analysis and properly accounting for the full cost in the multi-class setting of the non-budgeted systems, we find that the results are in fact more favorable.

Figure~\ref{fig:tdef} further shows that system accuracy varied minimally across trials for all binary datasets, as indicated by their negligible error bars. In contrast, multi-class datasets in Figure~\ref{fig:tdef-multi} displayed greater variation, particularly the \texttt{letter} dataset, which, with the largest number of classes among all tested datasets, exhibited wider error bars.

The fluctuations, or “jitters,” in the accuracy curves are not due to insufficient iterations but are inherent to the online, active learning process. They arise from stochastic components of the algorithm, including the random sampling of which expert to consider, the probabilistic decision of when to query, the dependency on individual data points, and the continual pruning of the hypothesis set.

Our empirical results suggest that some key components of the label complexity theoretical bound, the disagreement coefficient $\theta$, the slope asymmetry $K_\ell$, and the best-in-class loss $\sE^*$, exert only a limited effect, thereby substantially reducing the impact of the higher-order term $T$.

\ignore{
Our empirical observations, when juxtaposed with the theory, suggest that the constituent factors of the label complexity bound—namely the disagreement coefficient $\theta$, slope asymmetry $K_\ell$, and best-in-class loss $\sE^*$—are all of small magnitude, substantially attenuating the effect of the higher-order term $T$ within the bound.
}

\section{Conclusion}

Our results suggest that it is possible to train high-performing
deferral algorithms while substantially reducing the cost of expert
queries during training, thereby reconciling the practical constraints
of real-world systems with the theoretical promise of learning to
defer. This makes deferral strategies far more feasible in
resource-constrained environments, especially in applications
involving expensive experts such as large language models or human
annotators.  Our framework also opens the door to several possible
extensions, including adaptive querying strategies that exploit
structure across training examples, settings with dynamically
available or context-dependent experts, as well as integration with
reinforcement learning for sequential or interactive decision-making.

\bibliography{bdef}

\begin{thebibliography}{120}
\providecommand{\natexlab}[1]{#1}
\providecommand{\url}[1]{\texttt{#1}}
\expandafter\ifx\csname urlstyle\endcsname\relax
  \providecommand{\doi}[1]{doi: #1}\else
  \providecommand{\doi}{doi: \begingroup \urlstyle{rm}\Url}\fi

\bibitem[Acar et~al.(2020)Acar, Gangrade, and Saligrama]{acar2020budget}
Durmus Alp~Emre Acar, Aditya Gangrade, and Venkatesh Saligrama.
\newblock Budget learning via bracketing.
\newblock In \emph{International Conference on Artificial Intelligence and Statistics}, pages 4109--4119, 2020.

\bibitem[Agrawal and Devanur(2016)]{agrawal2016linear}
Shipra Agrawal and Nikhil Devanur.
\newblock Linear contextual bandits with knapsacks.
\newblock In \emph{Advances in neural information processing systems}, 2016.

\bibitem[Alves et~al.(2024)Alves, Leit{\~a}o, Jesus, Sampaio, Li{\'e}bana, Saleiro, Figueiredo, and Bizarro]{alvescost}
Jean~Vieira Alves, Diogo Leit{\~a}o, S{\'e}rgio Jesus, Marco~OP Sampaio, Javier Li{\'e}bana, Pedro Saleiro, Mario~AT Figueiredo, and Pedro Bizarro.
\newblock Cost-sensitive learning to defer to multiple experts with workload constraints.
\newblock \emph{Transactions on Machine Learning Research}, 2024.

\bibitem[Amin et~al.(2020)Amin, Cortes, DeSalvo, and Rostamizadeh]{amin2020understanding}
Kareem Amin, Corinna Cortes, Giulia DeSalvo, and Afshin Rostamizadeh.
\newblock Understanding the effects of batching in online active learning.
\newblock In \emph{International Conference on Artificial Intelligence and Statistics}, pages 3482--3492, 2020.

\bibitem[Awasthi et~al.(2021{\natexlab{a}})Awasthi, Frank, Mao, Mohri, and Zhong]{awasthi2021calibration}
Pranjal Awasthi, Natalie Frank, Anqi Mao, Mehryar Mohri, and Yutao Zhong.
\newblock Calibration and consistency of adversarial surrogate losses.
\newblock \emph{Advances in Neural Information Processing Systems}, pages 9804--9815, 2021{\natexlab{a}}.

\bibitem[Awasthi et~al.(2021{\natexlab{b}})Awasthi, Mao, Mohri, and Zhong]{awasthi2021finer}
Pranjal Awasthi, Anqi Mao, Mehryar Mohri, and Yutao Zhong.
\newblock A finer calibration analysis for adversarial robustness.
\newblock \emph{arXiv preprint arXiv:2105.01550}, 2021{\natexlab{b}}.

\bibitem[Awasthi et~al.(2022{\natexlab{a}})Awasthi, Mao, Mohri, and Zhong]{awasthi2022Hconsistency}
Pranjal Awasthi, Anqi Mao, Mehryar Mohri, and Yutao Zhong.
\newblock {${\mathscr H}$}-consistency bounds for surrogate loss minimizers.
\newblock In \emph{International Conference on Machine Learning}, 2022{\natexlab{a}}.

\bibitem[Awasthi et~al.(2022{\natexlab{b}})Awasthi, Mao, Mohri, and Zhong]{awasthi2022multi}
Pranjal Awasthi, Anqi Mao, Mehryar Mohri, and Yutao Zhong.
\newblock Multi-class {${\mathscr H}$}-consistency bounds.
\newblock In \emph{Advances in neural information processing systems}, 2022{\natexlab{b}}.

\bibitem[Awasthi et~al.(2023)Awasthi, Mao, Mohri, and Zhong]{AwasthiMaoMohriZhong2023theoretically}
Pranjal Awasthi, Anqi Mao, Mehryar Mohri, and Yutao Zhong.
\newblock Theoretically grounded loss functions and algorithms for adversarial robustness.
\newblock In \emph{International Conference on Artificial Intelligence and Statistics}, pages 10077--10094, 2023.

\bibitem[Awasthi et~al.(2024)Awasthi, Mao, Mohri, and Zhong]{awasthi2024dc}
Pranjal Awasthi, Anqi Mao, Mehryar Mohri, and Yutao Zhong.
\newblock {DC}-programming for neural network optimizations.
\newblock \emph{Journal of Global Optimization}, 2024.

\bibitem[Bartlett and Wegkamp(2008)]{bartlett2008classification}
Peter~L Bartlett and Marten~H Wegkamp.
\newblock Classification with a reject option using a hinge loss.
\newblock \emph{Journal of Machine Learning Research}, 9\penalty0 (8), 2008.

\bibitem[Benz and Rodriguez(2022)]{benz2022counterfactual}
Nina L~Corvelo Benz and Manuel~Gomez Rodriguez.
\newblock Counterfactual inference of second opinions.
\newblock In \emph{Uncertainty in Artificial Intelligence}, pages 453--463, 2022.

\bibitem[Berkson(1944)]{Berkson1944}
Joseph Berkson.
\newblock Application of the logistic function to bio-assay.
\newblock \emph{Journal of the American Statistical Association}, 39:\penalty0 357--365, 1944.

\bibitem[Berkson(1951)]{Berkson1951}
Joseph Berkson.
\newblock Why {I} prefer logits to probits.
\newblock \emph{Biometrics}, 7\penalty0 (4):\penalty0 327--339, 1951.

\bibitem[Beygelzimer et~al.(2009)Beygelzimer, Dasgupta, and Langford]{beygelzimer2009importance}
Alina Beygelzimer, Sanjoy Dasgupta, and John Langford.
\newblock Importance weighted active learning.
\newblock In \emph{International conference on machine learning}, pages 49--56, 2009.

\bibitem[Bubeck et~al.(2023)Bubeck, Chandrasekaran, Eldan, Gehrke, Horvitz, Kamar, Lee, Lee, Li, Lundberg, et~al.]{bubeck2023sparks}
S{\'e}bastien Bubeck, Varun Chandrasekaran, Ronen Eldan, Johannes Gehrke, Eric Horvitz, Ece Kamar, Peter Lee, Yin~Tat Lee, Yuanzhi Li, Scott Lundberg, et~al.
\newblock Sparks of artificial general intelligence: Early experiments with gpt-4.
\newblock \emph{arXiv preprint arXiv:2303.12712}, 2023.

\bibitem[Cao et~al.(2022)Cao, Cai, Feng, Gu, Gu, An, Niu, and Sugiyama]{caogeneralizing}
Yuzhou Cao, Tianchi Cai, Lei Feng, Lihong Gu, Jinjie Gu, Bo~An, Gang Niu, and Masashi Sugiyama.
\newblock Generalizing consistent multi-class classification with rejection to be compatible with arbitrary losses.
\newblock In \emph{Advances in neural information processing systems}, 2022.

\bibitem[Cao et~al.(2023)Cao, Mozannar, Feng, Wei, and An]{cao2023defense}
Yuzhou Cao, Hussein Mozannar, Lei Feng, Hongxin Wei, and Bo~An.
\newblock In defense of softmax parametrization for calibrated and consistent learning to defer.
\newblock In \emph{Advances in Neural Information Processing Systems}, 2023.

\bibitem[Cesa-Bianchi and Gentile(2008)]{cesa2008improved}
Nicolo Cesa-Bianchi and Claudio Gentile.
\newblock Improved risk tail bounds for on-line algorithms.
\newblock \emph{IEEE Transactions on Information Theory}, 54\penalty0 (1):\penalty0 386--390, 2008.

\bibitem[Charoenphakdee et~al.(2021)Charoenphakdee, Cui, Zhang, and Sugiyama]{charoenphakdee2021classification}
Nontawat Charoenphakdee, Zhenghang Cui, Yivan Zhang, and Masashi Sugiyama.
\newblock Classification with rejection based on cost-sensitive classification.
\newblock In \emph{International Conference on Machine Learning}, pages 1507--1517, 2021.

\bibitem[Charusaie and Samadi(2024)]{Mohammad2024}
Mohammad-Amin Charusaie and Samira Samadi.
\newblock A unifying post-processing framework for multi-objective learn-to-defer problems.
\newblock In \emph{Advances in Neural Information Processing Systems}, 2024.

\bibitem[Charusaie et~al.(2022)Charusaie, Mozannar, Sontag, and Samadi]{charusaie2022sample}
Mohammad-Amin Charusaie, Hussein Mozannar, David Sontag, and Samira Samadi.
\newblock Sample efficient learning of predictors that complement humans.
\newblock In \emph{International Conference on Machine Learning}, pages 2972--3005, 2022.

\bibitem[Chen et~al.(2024)Chen, Li, Sun, and Wang]{chen2024learning}
Guanting Chen, Xiaocheng Li, Chunlin Sun, and Hanzhao Wang.
\newblock Learning to make adherence-aware advice.
\newblock In \emph{International Conference on Learning Representations}, 2024.

\bibitem[Cheng et~al.(2023)Cheng, Cao, Wang, Wei, An, and Feng]{cheng2023regression}
Xin Cheng, Yuzhou Cao, Haobo Wang, Hongxin Wei, Bo~An, and Lei Feng.
\newblock Regression with cost-based rejection.
\newblock In \emph{Advances in Neural Information Processing Systems}, 2023.

\bibitem[Chow(1970)]{chow1970optimum}
C~Chow.
\newblock On optimum recognition error and reject tradeoff.
\newblock \emph{IEEE Transactions on Information Theory}, 16\penalty0 (1):\penalty0 41--46, 1970.

\bibitem[Chow(1957)]{Chow1957}
C.K. Chow.
\newblock An optimum character recognition system using decision function.
\newblock \emph{IEEE Transactions on Computers}, 1957.

\bibitem[Cortes et~al.(2016{\natexlab{a}})Cortes, DeSalvo, and Mohri]{CortesDeSalvoMohri2016}
Corinna Cortes, Giulia DeSalvo, and Mehryar Mohri.
\newblock Learning with rejection.
\newblock In \emph{International Conference on Algorithmic Learning Theory}, pages 67--82, 2016{\natexlab{a}}.

\bibitem[Cortes et~al.(2016{\natexlab{b}})Cortes, DeSalvo, and Mohri]{CortesDeSalvoMohri2016bis}
Corinna Cortes, Giulia DeSalvo, and Mehryar Mohri.
\newblock Boosting with abstention.
\newblock In \emph{Advances in Neural Information Processing Systems}, pages 1660--1668, 2016{\natexlab{b}}.

\bibitem[Cortes et~al.(2019{\natexlab{a}})Cortes, DeSalvo, Gentile, Mohri, and Zhang]{cortes2019disgraph}
Corinna Cortes, Giulia DeSalvo, Claudio Gentile, Mehryar Mohri, and Ningshan Zhang.
\newblock Active learning with disagreement graphs.
\newblock In \emph{International conference on machine learning}, 2019{\natexlab{a}}.

\bibitem[Cortes et~al.(2019{\natexlab{b}})Cortes, DeSalvo, Gentile, Mohri, and Zhang]{cortes2019rbal}
Corinna Cortes, Giulia DeSalvo, Claudio Gentile, Mehryar Mohri, and Ningshan Zhang.
\newblock Region-based active learning.
\newblock In \emph{International Conference on Artificial Intelligence and Statistics}, 2019{\natexlab{b}}.

\bibitem[Cortes et~al.(2020)Cortes, DeSalvo, Gentile, Mohri, and Zhang]{cortes2020adaptive}
Corinna Cortes, Giulia DeSalvo, Claudio Gentile, Mehryar Mohri, and Ningshan Zhang.
\newblock Adaptive region-based active learning.
\newblock In \emph{International Conference on Machine Learning}, pages 2144--2153, 2020.

\bibitem[Cortes et~al.(2023)Cortes, DeSalvo, and Mohri]{CortesDeSalvoMohri2023}
Corinna Cortes, Giulia DeSalvo, and Mehryar Mohri.
\newblock Theory and algorithms for learning with rejection in binary classification.
\newblock \emph{Annals of Mathematics and Artificial Intelligence}, pages 1--39, 2023.

\bibitem[Cortes et~al.(2024)Cortes, Mao, Mohri, Mohri, and Zhong]{cortes2024cardinality}
Corinna Cortes, Anqi Mao, Christopher Mohri, Mehryar Mohri, and Yutao Zhong.
\newblock Cardinality-aware set prediction and top-$ k $ classification.
\newblock In \emph{Advances in Neural Information Processing Systems}, 2024.

\bibitem[Cortes et~al.(2025{\natexlab{a}})Cortes, Mao, Mohri, and Zhong]{CortesMaoMohriZhong2025balancing}
Corinna Cortes, Anqi Mao, Mehryar Mohri, and Yutao Zhong.
\newblock Balancing the scales: A theoretical and algorithmic framework for learning from imbalanced data.
\newblock In \emph{International Conference on Machine Learning}, 2025{\natexlab{a}}.

\bibitem[Cortes et~al.(2025{\natexlab{b}})Cortes, Mohri, and Zhong]{cortes2025balanced}
Corinna Cortes, Mehryar Mohri, and Yutao Zhong.
\newblock Improved balanced classification with theoretically grounded loss functions.
\newblock In \emph{Advances in Neural Information Processing Systems}, 2025{\natexlab{b}}.

\bibitem[De et~al.(2020)De, Koley, Ganguly, and Gomez-Rodriguez]{de2020regression}
Abir De, Paramita Koley, Niloy Ganguly, and Manuel Gomez-Rodriguez.
\newblock Regression under human assistance.
\newblock In \emph{Proceedings of the AAAI Conference on Artificial Intelligence}, pages 2611--2620, 2020.

\bibitem[De et~al.(2021)De, Okati, Zarezade, and Rodriguez]{de2021classification}
Abir De, Nastaran Okati, Ali Zarezade, and Manuel~Gomez Rodriguez.
\newblock Classification under human assistance.
\newblock In \emph{Proceedings of the AAAI Conference on Artificial Intelligence}, pages 5905--5913, 2021.

\bibitem[El-Yaniv and Wiener(2012)]{el2012active}
Ran El-Yaniv and Yair Wiener.
\newblock Active learning via perfect selective classification.
\newblock \emph{Journal of Machine Learning Research}, 13\penalty0 (2), 2012.

\bibitem[El-Yaniv et~al.(2010)]{el2010foundations}
Ran El-Yaniv et~al.
\newblock On the foundations of noise-free selective classification.
\newblock \emph{Journal of Machine Learning Research}, 11\penalty0 (5), 2010.

\bibitem[Filippi et~al.(2010)Filippi, Cappe, Garivier, and Szepesv{\'a}ri]{filippi2010parametric}
Sarah Filippi, Olivier Cappe, Aur{\'e}lien Garivier, and Csaba Szepesv{\'a}ri.
\newblock Parametric bandits: The generalized linear case.
\newblock In \emph{Advances in neural information processing systems}, 2010.

\bibitem[Freedman(1975)]{freedman1975tail}
David~A Freedman.
\newblock On tail probabilities for martingales.
\newblock \emph{the Annals of Probability}, pages 100--118, 1975.

\bibitem[Gangrade et~al.(2021)Gangrade, Kag, and Saligrama]{gangrade2021selective}
Aditya Gangrade, Anil Kag, and Venkatesh Saligrama.
\newblock Selective classification via one-sided prediction.
\newblock In \emph{International Conference on Artificial Intelligence and Statistics}, pages 2179--2187, 2021.

\bibitem[Gao et~al.(2021)Gao, Saar-Tsechansky, De-Arteaga, Han, Lee, and Lease]{gao2021human}
Ruijiang Gao, Maytal Saar-Tsechansky, Maria De-Arteaga, Ligong Han, Min~Kyung Lee, and Matthew Lease.
\newblock Human-ai collaboration with bandit feedback.
\newblock \emph{arXiv preprint arXiv:2105.10614}, 2021.

\bibitem[Geifman and El-Yaniv(2017)]{geifman2017selective}
Yonatan Geifman and Ran El-Yaniv.
\newblock Selective classification for deep neural networks.
\newblock In \emph{Advances in Neural Information Processing Systems}, 2017.

\bibitem[Geifman and El-Yaniv(2019)]{geifman2019selectivenet}
Yonatan Geifman and Ran El-Yaniv.
\newblock Selectivenet: A deep neural network with an integrated reject option.
\newblock In \emph{International Conference on Machine Learning}, pages 2151--2159, 2019.

\bibitem[Hanneke(2007)]{hanneke2007bound}
Steve Hanneke.
\newblock A bound on the label complexity of agnostic active learning.
\newblock In \emph{International Conference on Machine Learning}, pages 353--360, 2007.

\bibitem[Hanneke(2014)]{Hanneke2014}
Steve Hanneke.
\newblock Theory of disagreement-based active learning.
\newblock \emph{Found. Trends Mach. Learn.}, 7\penalty0 (2-3):\penalty0 131--309, 2014.

\bibitem[Hemmer et~al.(2022)Hemmer, Schellhammer, V{\"o}ssing, Jakubik, and Satzger]{hemmer2022forming}
Patrick Hemmer, Sebastian Schellhammer, Michael V{\"o}ssing, Johannes Jakubik, and Gerhard Satzger.
\newblock Forming effective human-ai teams: Building machine learning models that complement the capabilities of multiple experts.
\newblock \emph{arXiv preprint arXiv:2206.07948}, 2022.

\bibitem[Hemmer et~al.(2023)Hemmer, Thede, V{\"o}ssing, Jakubik, and K{\"u}hl]{hemmer2023learning}
Patrick Hemmer, Lukas Thede, Michael V{\"o}ssing, Johannes Jakubik, and Niklas K{\"u}hl.
\newblock Learning to defer with limited expert predictions.
\newblock \emph{arXiv preprint arXiv:2304.07306}, 2023.

\bibitem[Jiang et~al.(2020)Jiang, Zhao, and Wang]{jiang2020risk}
Wenming Jiang, Ying Zhao, and Zehan Wang.
\newblock Risk-controlled selective prediction for regression deep neural network models.
\newblock In \emph{International Joint Conference on Neural Networks}, pages 1--8, 2020.

\bibitem[Jitkrittum et~al.(2023)Jitkrittum, Gupta, Menon, Narasimhan, Rawat, and Kumar]{jitkrittum2023does}
Wittawat Jitkrittum, Neha Gupta, Aditya~K Menon, Harikrishna Narasimhan, Ankit Rawat, and Sanjiv Kumar.
\newblock When does confidence-based cascade deferral suffice?
\newblock In \emph{Advances in Neural Information Processing Systems}, 2023.

\bibitem[Joshi et~al.(2021)Joshi, Parbhoo, and Doshi-Velez]{joshi2021pre}
Shalmali Joshi, Sonali Parbhoo, and Finale Doshi-Velez.
\newblock Pre-emptive learning-to-defer for sequential medical decision-making under uncertainty.
\newblock \emph{arXiv preprint arXiv:2109.06312}, 2021.

\bibitem[Kakade and Tewari(2008)]{kakade2008generalization}
Sham~M Kakade and Ambuj Tewari.
\newblock On the generalization ability of online strongly convex programming algorithms.
\newblock In \emph{Advances in neural information processing systems}, 2008.

\bibitem[Kerrigan et~al.(2021)Kerrigan, Smyth, and Steyvers]{kerrigan2021combining}
Gavin Kerrigan, Padhraic Smyth, and Mark Steyvers.
\newblock Combining human predictions with model probabilities via confusion matrices and calibration.
\newblock In \emph{Advances in Neural Information Processing Systems}, pages 4421--4434, 2021.

\bibitem[Keswani et~al.(2021)Keswani, Lease, and Kenthapadi]{keswani2021towards}
Vijay Keswani, Matthew Lease, and Krishnaram Kenthapadi.
\newblock Towards unbiased and accurate deferral to multiple experts.
\newblock In \emph{Proceedings of the AAAI/ACM Conference on AI, Ethics, and Society}, pages 154--165, 2021.

\bibitem[Lattimore and Szepesv{\'a}ri(2020)]{lattimore2020bandit}
Tor Lattimore and Csaba Szepesv{\'a}ri.
\newblock \emph{Bandit algorithms}.
\newblock Cambridge University Press, 2020.

\bibitem[Li et~al.(2017)Li, Lu, and Zhou]{li2017provably}
Lihong Li, Yu~Lu, and Dengyong Zhou.
\newblock Provably optimal algorithms for generalized linear contextual bandits.
\newblock In \emph{International Conference on Machine Learning}, pages 2071--2080, 2017.

\bibitem[Li et~al.(2023)Li, Liu, Sun, and Wang]{li2023no}
Xiaocheng Li, Shang Liu, Chunlin Sun, and Hanzhao Wang.
\newblock When no-rejection learning is optimal for regression with rejection.
\newblock \emph{arXiv preprint arXiv:2307.02932}, 2023.

\bibitem[Liu et~al.(2024)Liu, Cao, Zhang, Feng, and An]{liu2024mitigating}
Shuqi Liu, Yuzhou Cao, Qiaozhen Zhang, Lei Feng, and Bo~An.
\newblock Mitigating underfitting in learning to defer with consistent losses.
\newblock In \emph{International Conference on Artificial Intelligence and Statistics}, pages 4816--4824, 2024.

\bibitem[Madras et~al.(2018)Madras, Creager, Pitassi, and Zemel]{madras2018learning}
David Madras, Elliot Creager, Toniann Pitassi, and Richard Zemel.
\newblock Learning adversarially fair and transferable representations.
\newblock \emph{arXiv preprint arXiv:1802.06309}, 2018.

\bibitem[Mao(2025)]{mao2025theory}
Anqi Mao.
\newblock \emph{Theory and Algorithms for Learning with Multi-Class Abstention and Multi-Expert Deferral}.
\newblock PhD thesis, New York University, 2025.

\bibitem[Mao et~al.(2023{\natexlab{a}})Mao, Mohri, Mohri, and Zhong]{MaoMohriMohriZhong2023two}
Anqi Mao, Christopher Mohri, Mehryar Mohri, and Yutao Zhong.
\newblock Two-stage learning to defer with multiple experts.
\newblock In \emph{Advances in Neural Information Processing Systems}, 2023{\natexlab{a}}.

\bibitem[Mao et~al.(2023{\natexlab{b}})Mao, Mohri, and Zhong]{MaoMohriZhong2023characterization}
Anqi Mao, Mehryar Mohri, and Yutao Zhong.
\newblock {H}-consistency bounds: Characterization and extensions.
\newblock In \emph{Advances in Neural Information Processing Systems}, 2023{\natexlab{b}}.

\bibitem[Mao et~al.(2023{\natexlab{c}})Mao, Mohri, and Zhong]{MaoMohriZhong2023ranking}
Anqi Mao, Mehryar Mohri, and Yutao Zhong.
\newblock {H}-consistency bounds for pairwise misranking loss surrogates.
\newblock In \emph{International conference on Machine learning}, 2023{\natexlab{c}}.

\bibitem[Mao et~al.(2023{\natexlab{d}})Mao, Mohri, and Zhong]{MaoMohriZhong2023rankingabs}
Anqi Mao, Mehryar Mohri, and Yutao Zhong.
\newblock Ranking with abstention.
\newblock In \emph{ICML 2023 Workshop The Many Facets of Preference-Based Learning}, 2023{\natexlab{d}}.

\bibitem[Mao et~al.(2023{\natexlab{e}})Mao, Mohri, and Zhong]{MaoMohriZhong2023structured}
Anqi Mao, Mehryar Mohri, and Yutao Zhong.
\newblock Structured prediction with stronger consistency guarantees.
\newblock In \emph{Advances in Neural Information Processing Systems}, 2023{\natexlab{e}}.

\bibitem[Mao et~al.(2023{\natexlab{f}})Mao, Mohri, and Zhong]{mao2023cross}
Anqi Mao, Mehryar Mohri, and Yutao Zhong.
\newblock Cross-entropy loss functions: Theoretical analysis and applications.
\newblock In \emph{International Conference on Machine Learning}, 2023{\natexlab{f}}.

\bibitem[Mao et~al.(2024{\natexlab{a}})Mao, Mohri, and Zhong]{MaoMohriZhong2024deferral}
Anqi Mao, Mehryar Mohri, and Yutao Zhong.
\newblock Principled approaches for learning to defer with multiple experts.
\newblock In \emph{International Symposium on Artificial Intelligence and Mathematics}, 2024{\natexlab{a}}.

\bibitem[Mao et~al.(2024{\natexlab{b}})Mao, Mohri, and Zhong]{MaoMohriZhong2024predictor}
Anqi Mao, Mehryar Mohri, and Yutao Zhong.
\newblock Predictor-rejector multi-class abstention: Theoretical analysis and algorithms.
\newblock In \emph{International Conference on Algorithmic Learning Theory}, pages 822--867, 2024{\natexlab{b}}.

\bibitem[Mao et~al.(2024{\natexlab{c}})Mao, Mohri, and Zhong]{MaoMohriZhong2024score}
Anqi Mao, Mehryar Mohri, and Yutao Zhong.
\newblock Theoretically grounded loss functions and algorithms for score-based multi-class abstention.
\newblock In \emph{International Conference on Artificial Intelligence and Statistics}, pages 4753--4761, 2024{\natexlab{c}}.

\bibitem[Mao et~al.(2024{\natexlab{d}})Mao, Mohri, and Zhong]{mao2024h}
Anqi Mao, Mehryar Mohri, and Yutao Zhong.
\newblock {$ H $}-consistency guarantees for regression.
\newblock In \emph{International Conference on Machine Learning}, pages 34712--34737, 2024{\natexlab{d}}.

\bibitem[Mao et~al.(2024{\natexlab{e}})Mao, Mohri, and Zhong]{mao2024multi}
Anqi Mao, Mehryar Mohri, and Yutao Zhong.
\newblock Multi-label learning with stronger consistency guarantees.
\newblock In \emph{Advances in Neural Information Processing Systems}, 2024{\natexlab{e}}.

\bibitem[Mao et~al.(2024{\natexlab{f}})Mao, Mohri, and Zhong]{mao2024regression}
Anqi Mao, Mehryar Mohri, and Yutao Zhong.
\newblock Regression with multi-expert deferral.
\newblock In \emph{International Conference on Machine Learning}, 2024{\natexlab{f}}.

\bibitem[Mao et~al.(2024{\natexlab{g}})Mao, Mohri, and Zhong]{mao2024universal}
Anqi Mao, Mehryar Mohri, and Yutao Zhong.
\newblock A universal growth rate for learning with smooth surrogate losses.
\newblock In \emph{Advances in Neural Information Processing Systems}, 2024{\natexlab{g}}.

\bibitem[Mao et~al.(2024{\natexlab{h}})Mao, Mohri, and Zhong]{maorealizable}
Anqi Mao, Mehryar Mohri, and Yutao Zhong.
\newblock Realizable {$H$}-consistent and {B}ayes-consistent loss functions for learning to defer.
\newblock In \emph{Advances in Neural Information Processing Systems}, 2024{\natexlab{h}}.

\bibitem[Mao et~al.(2025{\natexlab{a}})Mao, Mohri, and Zhong]{MaoMohriZhong2025}
Anqi Mao, Mehryar Mohri, and Yutao Zhong.
\newblock Mastering multiple-expert routing: Realizable {$H$}-consistency and strong guarantees for learning to defer.
\newblock In \emph{International Conference on Machine Learning}, 2025{\natexlab{a}}.

\bibitem[Mao et~al.(2025{\natexlab{b}})Mao, Mohri, and Zhong]{MaoMohriZhong2025principled}
Anqi Mao, Mehryar Mohri, and Yutao Zhong.
\newblock Principled algorithms for optimizing generalized metrics in binary classification.
\newblock In \emph{International Conference on Machine Learning}, 2025{\natexlab{b}}.

\bibitem[Mao et~al.(2025{\natexlab{c}})Mao, Mohri, and Zhong]{mao2025enhanced}
Anqi Mao, Mehryar Mohri, and Yutao Zhong.
\newblock Enhanced {$\sH $}-consistency bounds.
\newblock In \emph{International Conference on Algorithmic Learning Theory}, 2025{\natexlab{c}}.

\bibitem[Mohri et~al.(2024)Mohri, Andor, Choi, Collins, Mao, and Zhong]{MohriAndorChoiCollinsMaoZhong2024learning}
Christopher Mohri, Daniel Andor, Eunsol Choi, Michael Collins, Anqi Mao, and Yutao Zhong.
\newblock Learning to reject with a fixed predictor: Application to decontextualization.
\newblock In \emph{International Conference on Learning Representations}, 2024.

\bibitem[Mohri and Zhong(2026)]{MaoMohriZhong2026model}
Mehryar Mohri and Yutao Zhong.
\newblock Model margin noise and {$H$}-consistency bounds.
\newblock In \emph{International Symposium on Artificial Intelligence and Mathematics}, 2026.

\bibitem[Mohri et~al.(2018)Mohri, Rostamizadeh, and Talwalkar]{MohriRostamizadehTalwalkar2018}
Mehryar Mohri, Afshin Rostamizadeh, and Ameet Talwalkar.
\newblock \emph{Foundations of Machine Learning}.
\newblock {MIT} Press, second edition, 2018.

\bibitem[Montreuil et~al.(2024)Montreuil, Yeo, Carlier, Ng, and Ooi]{montreuil2024learning}
Yannis Montreuil, Shu~Heng Yeo, Axel Carlier, Lai~Xing Ng, and Wei~Tsang Ooi.
\newblock Optimal query allocation in extractive {QA} with {LLMs}: A learning-to-defer framework with theoretical guarantees.
\newblock \emph{arXiv preprint arXiv:2410.15761}, 2024.

\bibitem[Montreuil et~al.(2025{\natexlab{a}})Montreuil, Carlier, Ng, and Ooi]{montreuil2025adversarial}
Yannis Montreuil, Axel Carlier, Lai~Xing Ng, and Wei~Tsang Ooi.
\newblock Adversarial robustness in two-stage learning-to-defer: Algorithms and guarantees.
\newblock In \emph{International Conference on Machine Learning}, 2025{\natexlab{a}}.

\bibitem[Montreuil et~al.(2025{\natexlab{b}})Montreuil, Carlier, Ng, and Ooi]{montreuil2025ask}
Yannis Montreuil, Axel Carlier, Lai~Xing Ng, and Wei~Tsang Ooi.
\newblock Why ask one when you can ask $ k $? two-stage learning-to-defer to the top-$ k $ experts.
\newblock \emph{arXiv preprint arXiv:2504.12988}, 2025{\natexlab{b}}.

\bibitem[Montreuil et~al.(2025{\natexlab{c}})Montreuil, Carlier, Ng, and Ooi]{montreuil2025one}
Yannis Montreuil, Axel Carlier, Lai~Xing Ng, and Wei~Tsang Ooi.
\newblock One-stage top-$ k $ learning-to-defer: Score-based surrogates with theoretical guarantees.
\newblock \emph{arXiv preprint arXiv:2505.10160}, 2025{\natexlab{c}}.

\bibitem[Montreuil et~al.(2025{\natexlab{d}})Montreuil, Yeo, Carlier, Ng, and Ooi]{montreuil2025two}
Yannis Montreuil, Shu~Heng Yeo, Axel Carlier, Lai~Xing Ng, and Wei~Tsang Ooi.
\newblock A two-stage learning-to-defer approach for multi-task learning.
\newblock In \emph{International Conference on Machine Learning}, 2025{\natexlab{d}}.

\bibitem[Mozannar and Sontag(2020)]{mozannar2020consistent}
Hussein Mozannar and David Sontag.
\newblock Consistent estimators for learning to defer to an expert.
\newblock In \emph{International Conference on Machine Learning}, pages 7076--7087, 2020.

\bibitem[Mozannar et~al.(2022)Mozannar, Satyanarayan, and Sontag]{mozannar2022teaching}
Hussein Mozannar, Arvind Satyanarayan, and David Sontag.
\newblock Teaching humans when to defer to a classifier via exemplars.
\newblock In \emph{Proceedings of the AAAI Conference on Artificial Intelligence}, pages 5323--5331, 2022.

\bibitem[Mozannar et~al.(2023)Mozannar, Lang, Wei, Sattigeri, Das, and Sontag]{pmlr-v206-mozannar23a}
Hussein Mozannar, Hunter Lang, Dennis Wei, Prasanna Sattigeri, Subhro Das, and David Sontag.
\newblock Who should predict? exact algorithms for learning to defer to humans.
\newblock In \emph{International Conference on Artificial Intelligence and Statistics}, pages 10520--10545, 2023.

\bibitem[Narasimhan et~al.(2022)Narasimhan, Jitkrittum, Menon, Rawat, and Kumar]{narasimhanpost}
Harikrishna Narasimhan, Wittawat Jitkrittum, Aditya~Krishna Menon, Ankit~Singh Rawat, and Sanjiv Kumar.
\newblock Post-hoc estimators for learning to defer to an expert.
\newblock In \emph{Advances in Neural Information Processing Systems}, pages 29292--29304, 2022.

\bibitem[Narasimhan et~al.(2024)Narasimhan, Menon, Jitkrittum, Gupta, and Kumar]{narasimhanlearning}
Harikrishna Narasimhan, Aditya~Krishna Menon, Wittawat Jitkrittum, Neha Gupta, and Sanjiv Kumar.
\newblock Learning to reject meets long-tail learning.
\newblock In \emph{International Conference on Learning Representations}, 2024.

\bibitem[Neyman and Pearson(1933)]{neyman1933ix}
Jerzy Neyman and Egon~Sharpe Pearson.
\newblock Ix. on the problem of the most efficient tests of statistical hypotheses.
\newblock \emph{Philosophical Transactions of the Royal Society of London. Series A, Containing Papers of a Mathematical or Physical Character}, 231\penalty0 (694-706):\penalty0 289--337, 1933.

\bibitem[Ni et~al.(2019)Ni, Charoenphakdee, Honda, and Sugiyama]{NiCHS19}
Chenri Ni, Nontawat Charoenphakdee, Junya Honda, and Masashi Sugiyama.
\newblock On the calibration of multiclass classification with rejection.
\newblock In \emph{Advances in Neural Information Processing Systems}, pages 2582--2592, 2019.

\bibitem[Okati et~al.(2021)Okati, De, and Rodriguez]{okati2021differentiable}
Nastaran Okati, Abir De, and Manuel Rodriguez.
\newblock Differentiable learning under triage.
\newblock In \emph{Advances in Neural Information Processing Systems}, pages 9140--9151, 2021.

\bibitem[Palomba et~al.(2024)Palomba, Pugnana, Alvarez, and Ruggieri]{palomba2024causal}
Filippo Palomba, Andrea Pugnana, Jos{\'e}~Manuel Alvarez, and Salvatore Ruggieri.
\newblock A causal framework for evaluating deferring systems.
\newblock \emph{arXiv preprint arXiv:2405.18902}, 2024.

\bibitem[Pedregosa et~al.(2011)Pedregosa, Varoquaux, Gramfort, Michel, Thirion, Grisel, Blondel, Prettenhofer, Weiss, Dubourg, Vanderplas, Passos, Cournapeau, Brucher, Perrot, and Duchesnay]{scikit-learn}
F.~Pedregosa, G.~Varoquaux, A.~Gramfort, V.~Michel, B.~Thirion, O.~Grisel, M.~Blondel, P.~Prettenhofer, R.~Weiss, V.~Dubourg, J.~Vanderplas, A.~Passos, D.~Cournapeau, M.~Brucher, M.~Perrot, and E.~Duchesnay.
\newblock Scikit-learn: Machine learning in {P}ython.
\newblock \emph{Journal of Machine Learning Research}, 12:\penalty0 2825--2830, 2011.

\bibitem[Pradier et~al.(2021)Pradier, Zazo, Parbhoo, Perlis, Zazzi, and Doshi-Velez]{pradier2021preferential}
Melanie~F Pradier, Javier Zazo, Sonali Parbhoo, Roy~H Perlis, Maurizio Zazzi, and Finale Doshi-Velez.
\newblock Preferential mixture-of-experts: Interpretable models that rely on human expertise as much as possible.
\newblock \emph{AMIA Summits on Translational Science Proceedings}, 2021:\penalty0 525, 2021.

\bibitem[Raghu et~al.(2019)Raghu, Blumer, Corrado, Kleinberg, Obermeyer, and Mullainathan]{raghu2019algorithmic}
Maithra Raghu, Katy Blumer, Greg Corrado, Jon Kleinberg, Ziad Obermeyer, and Sendhil Mullainathan.
\newblock The algorithmic automation problem: Prediction, triage, and human effort.
\newblock \emph{arXiv preprint arXiv:1903.12220}, 2019.

\bibitem[Ramaswamy et~al.(2018)Ramaswamy, Tewari, and Agarwal]{ramaswamy2018consistent}
Harish~G Ramaswamy, Ambuj Tewari, and Shivani Agarwal.
\newblock Consistent algorithms for multiclass classification with an abstain option.
\newblock \emph{Electronic Journal of Statistics}, 12\penalty0 (1):\penalty0 530--554, 2018.

\bibitem[Reid et~al.(2024)Reid, S{\"u}hr, Vernade, and Samadi]{reid2024online}
Mirabel Reid, Tom S{\"u}hr, Claire Vernade, and Samira Samadi.
\newblock Online decision deferral under budget constraints.
\newblock \emph{arXiv preprint arXiv:2409.20489}, 2024.

\bibitem[Shah et~al.(2022)Shah, Bu, Lee, Das, Panda, Sattigeri, and Wornell]{shah2022selective}
Abhin Shah, Yuheng Bu, Joshua~K Lee, Subhro Das, Rameswar Panda, Prasanna Sattigeri, and Gregory~W Wornell.
\newblock Selective regression under fairness criteria.
\newblock In \emph{International Conference on Machine Learning}, pages 19598--19615, 2022.

\bibitem[Straitouri et~al.(2021)Straitouri, Singla, Meresht, and Gomez-Rodriguez]{straitouri2021reinforcement}
Eleni Straitouri, Adish Singla, Vahid~Balazadeh Meresht, and Manuel Gomez-Rodriguez.
\newblock Reinforcement learning under algorithmic triage.
\newblock \emph{arXiv preprint arXiv:2109.11328}, 2021.

\bibitem[Straitouri et~al.(2022)Straitouri, Wang, Okati, and Rodriguez]{straitouri2022provably}
Eleni Straitouri, Lequn Wang, Nastaran Okati, and Manuel~Gomez Rodriguez.
\newblock Provably improving expert predictions with conformal prediction.
\newblock \emph{arXiv preprint arXiv:2201.12006}, 2022.

\bibitem[Tailor et~al.(2024)Tailor, Patra, Verma, Manggala, and Nalisnick]{tailor2024learning}
Dharmesh Tailor, Aditya Patra, Rajeev Verma, Putra Manggala, and Eric Nalisnick.
\newblock Learning to defer to a population: A meta-learning approach.
\newblock In \emph{International Conference on Artificial Intelligence and Statistics}, pages 3475--3483, 2024.

\bibitem[Verhulst(1838)]{Verhulst1838}
Pierre~François Verhulst.
\newblock Notice sur la loi que la population suit dans son accroissement.
\newblock \emph{Correspondance math\'ematique et physique}, 10:\penalty0 113--121, 1838.

\bibitem[Verhulst(1845)]{Verhulst1845}
Pierre~François Verhulst.
\newblock Recherches math\'ematiques sur la loi d'accroissement de la population.
\newblock \emph{Nouveaux M\'emoires de l'Acad\'emie Royale des Sciences et Belles-Lettres de Bruxelles}, 18:\penalty0 1--42, 1845.

\bibitem[Verma and Nalisnick(2022)]{verma2022calibrated}
Rajeev Verma and Eric Nalisnick.
\newblock Calibrated learning to defer with one-vs-all classifiers.
\newblock In \emph{International Conference on Machine Learning}, pages 22184--22202, 2022.

\bibitem[Verma et~al.(2023)Verma, Barrej{\'o}n, and Nalisnick]{verma2023learning}
Rajeev Verma, Daniel Barrej{\'o}n, and Eric Nalisnick.
\newblock Learning to defer to multiple experts: Consistent surrogate losses, confidence calibration, and conformal ensembles.
\newblock In \emph{International Conference on Artificial Intelligence and Statistics}, pages 11415--11434, 2023.

\bibitem[Wei et~al.(2022)Wei, Tay, Bommasani, Raffel, Zoph, Borgeaud, Yogatama, Bosma, Zhou, Metzler, Chi, Hashimoto, Vinyals, Liang, Dean, and Fedus]{WeiEtAl2022}
Jason Wei, Yi~Tay, Rishi Bommasani, Colin Raffel, Barret Zoph, Sebastian Borgeaud, Dani Yogatama, Maarten Bosma, Denny Zhou, Donald Metzler, Ed~H. Chi, Tatsunori Hashimoto, Oriol Vinyals, Percy Liang, Jeff Dean, and William Fedus.
\newblock Emergent abilities of large language models.
\newblock \emph{CoRR}, abs/2206.07682, 2022.

\bibitem[Wei et~al.(2024)Wei, Cao, and Feng]{weiexploiting}
Zixi Wei, Yuzhou Cao, and Lei Feng.
\newblock Exploiting human-ai dependence for learning to defer.
\newblock In \emph{International Conference on Machine Learning}, 2024.

\bibitem[Wiener and El-Yaniv(2011)]{wiener2011agnostic}
Yair Wiener and Ran El-Yaniv.
\newblock Agnostic selective classification.
\newblock In \emph{Advances in Neural Information Processing Systems}, 2011.

\bibitem[Wiener and El-Yaniv(2012)]{wiener2012pointwise}
Yair Wiener and Ran El-Yaniv.
\newblock Pointwise tracking the optimal regression function.
\newblock In \emph{Advances in Neural Information Processing Systems}, 2012.

\bibitem[Wiener and El-Yaniv(2015)]{wiener2015agnostic}
Yair Wiener and Ran El-Yaniv.
\newblock Agnostic pointwise-competitive selective classification.
\newblock \emph{Journal of Artificial Intelligence Research}, 52:\penalty0 171--201, 2015.

\bibitem[Wilder et~al.(2021)Wilder, Horvitz, and Kamar]{wilder2021learning}
Bryan Wilder, Eric Horvitz, and Ece Kamar.
\newblock Learning to complement humans.
\newblock In \emph{International Joint Conferences on Artificial Intelligence}, pages 1526--1533, 2021.

\bibitem[Yuan and Wegkamp(2010)]{yuan2010classification}
Ming Yuan and Marten Wegkamp.
\newblock Classification methods with reject option based on convex risk minimization.
\newblock \emph{Journal of Machine Learning Research}, 11\penalty0 (1), 2010.

\bibitem[Yuan and Wegkamp(2011)]{WegkampYuan2011}
Ming Yuan and Marten Wegkamp.
\newblock {SVM}s with a reject option.
\newblock In \emph{Bernoulli}, 2011.

\bibitem[Zaoui et~al.(2020)Zaoui, Denis, and Hebiri]{zaoui2020regression}
Ahmed Zaoui, Christophe Denis, and Mohamed Hebiri.
\newblock Regression with reject option and application to knn.
\newblock In \emph{Advances in Neural Information Processing Systems}, pages 20073--20082, 2020.

\bibitem[Zhang et~al.(2024)Zhang, Nguyen, Wells, Do, Rosewarne, and Carneiro]{zhang2024coverage}
Zheng Zhang, Cuong Nguyen, Kevin Wells, Thanh-Toan Do, David Rosewarne, and Gustavo Carneiro.
\newblock Coverage-constrained human-ai cooperation with multiple experts.
\newblock \emph{arXiv preprint arXiv:2411.11976}, 2024.

\bibitem[Zhao et~al.(2021)Zhao, Agrawal, Razavi, and Sontag]{zhao2021directing}
Jason Zhao, Monica Agrawal, Pedram Razavi, and David Sontag.
\newblock Directing human attention in event localization for clinical timeline creation.
\newblock In \emph{Machine Learning for Healthcare Conference}, pages 80--102, 2021.

\bibitem[Zhong(2025)]{zhong2025fundamental}
Yutao Zhong.
\newblock \emph{Fundamental Novel Consistency Theory: H-Consistency Bounds}.
\newblock PhD thesis, New York University, 2025.

\end{thebibliography}

\newpage
\appendix

\renewcommand{\contentsname}{Contents of Appendix}
\tableofcontents
\addtocontents{toc}{\protect\setcounter{tocdepth}{3}} 
\clearpage

\section{Related Work}
\label{app:related-work}

The paradigm of single-stage learning to defer, where a predictor and a deferral function are jointly trained, was established by \citet*{CortesDeSalvoMohri2016,CortesDeSalvoMohri2016bis, CortesDeSalvoMohri2023} and has since spurred considerable research. This includes foundational studies on abstention with constant costs \citep{MaoMohriZhong2024predictor,MaoMohriZhong2024score,MohriAndorChoiCollinsMaoZhong2024learning,caogeneralizing,charoenphakdee2021classification,cheng2023regression,li2023no,narasimhanlearning} and, more aligned with complex real-world scenarios, deferral involving instance- and label-dependent costs \citep{MaoMohriZhong2024deferral,cao2023defense,maorealizable,mozannar2020consistent,pmlr-v206-mozannar23a,verma2022calibrated,verma2023learning,weiexploiting}. In this established setup, the deferral function's primary role is to optimally assign input instances to the most suitable expert. While this L2D paradigm offers distinct advantages over traditional \emph{confidence-based} rejection methods \citep{Chow1957,NiCHS19,WegkampYuan2011,bartlett2008classification,chow1970optimum,jitkrittum2023does,ramaswamy2018consistent,yuan2010classification} and \emph{selective classification} techniques that typically use fixed selection rates and cannot adapt to expert-modeled cost functions \citep{acar2020budget,el2010foundations,el2012active,gangrade2021selective,geifman2017selective,geifman2019selectivenet,jiang2020risk,shah2022selective,wiener2011agnostic,wiener2012pointwise,wiener2015agnostic,zaoui2020regression}, these studies largely operate under the assumption that expert costs, even if variable, are specified or can be readily queried to train the deferral model. The distinct challenge of actively managing a budget for querying these expert costs during the learning phase of the deferral mechanism itself generally remains outside their primary scope, a research gap our work specifically targets.

The \emph{learning to defer} (L2D) problem, which integrates human expert decisions into the cost function, was initially framed by \citet{madras2018learning} and has been extensively studied \citep{MaoMohriMohriZhong2023two,MaoMohriZhong2024deferral,charusaie2022sample,maorealizable,mozannar2020consistent,pmlr-v206-mozannar23a,pradier2021preferential,raghu2019algorithmic,verma2022calibrated,wilder2021learning}. A core tenet is the formulation of a deferral loss function incorporating instance-specific costs for each expert. However, direct optimization of this loss is often intractable for commonly used hypothesis sets. Consequently, L2D algorithms typically resort to optimizing surrogate loss functions. This naturally raises the crucial question of the theoretical guarantees associated with such surrogate-based optimization.

This question, centered on the consistency guarantees of surrogate losses with respect to the target deferral loss, has been analyzed under two primary scenarios \citep{mao2025theory}: the \emph{single-stage} scenario, where predictor and deferral functions are learned jointly \citep{MaoMohriZhong2024deferral,charusaie2022sample,mozannar2020consistent,pmlr-v206-mozannar23a,verma2022calibrated}, and a \emph{two-stage} scenario, where a pre-trained predictor (acting as an expert) is fixed, and only the deferral function is subsequently learned \citep{MaoMohriMohriZhong2023two}. While these consistency analyses are vital for understanding the reliability of L2D approaches, they generally presuppose that the necessary label and cost information for training and evaluating consistency is either given or can be acquired without explicit budgetary constraints that fundamentally shape the learning algorithm's design. Our work diverges by concentrating on scenarios where the act of querying expert costs is itself a budgeted component integral to the training loop, necessitating a framework that explicitly manages this cost-acquisition process.

In particular, for the single-stage \emph{single-expert} case, \citet{mozannar2020consistent}, \citet{verma2022calibrated}, and \citet{charusaie2022sample} proposed surrogate losses by generalizing standard classification losses. A key development by \citet{pmlr-v206-mozannar23a} later revealed that these surrogates did not satisfy \emph{realizable $\sH$-consistency}, leading them to suggest an alternative. This was further refined by \citet{maorealizable}, who introduced a broader family of surrogates achieving Bayes-consistency, realizable $\sH$-consistency, and \emph{$\sH$-consistency bounds} \citep{awasthi2022Hconsistency,awasthi2022multi,mao2023cross}  (see also
\citep{awasthi2021calibration,awasthi2021finer,AwasthiMaoMohriZhong2023theoretically,awasthi2024dc,MaoMohriZhong2023characterization,MaoMohriZhong2023structured,MaoMohriZhong2023ranking,MaoMohriZhong2023rankingabs,mao2024h,mao2024multi,mao2024universal,cortes2024cardinality,mao2025enhanced,MaoMohriZhong2025principled,zhong2025fundamental,CortesMaoMohriZhong2025balancing,cortes2025balanced,MaoMohriZhong2026model}). For the single-stage \emph{multiple-expert} setting \citep{benz2022counterfactual,hemmer2022forming,kerrigan2021combining,keswani2021towards,straitouri2022provably}, \citet{verma2023learning} and \citet{MaoMohriZhong2024deferral} extended earlier surrogate losses. However, these extensions often inherit limitations regarding realizable $\sH$-consistency \citep{MaoMohriZhong2025}. In the \emph{two-stage} scenario, \citet{MaoMohriMohriZhong2023two} introduced surrogate losses with strong consistency properties for constant costs, though their realizable $\sH$-consistency does not straightforwardly extend to practical classification error-based costs. These important theoretical investigations primarily concern the fidelity of surrogate losses, assuming the data (including expert costs) for loss computation is accessible for training. They do not typically address the strategic, budgeted acquisition of this costly expert information, which is a central focus of our research. Other extensions cover regression \citep{mao2024regression}, deferral to populations \citep{tailor2024learning}, multi-task learning \citep{montreuil2025two}, and aspects of adversarial robustness \citep{montreuil2025adversarial}, top-$k$ learning \citep{montreuil2025ask,montreuil2025one} or specialized query mechanisms \citep{montreuil2024learning}, which, while related to querying, often adopt different problem formulations from our specific emphasis on minimizing expert query costs during the primary training of the deferral model under a budget.

The landscape of L2D extends beyond initial model training, with considerable work focusing on post-hoc refinements and deeper theoretical explorations. For instance, some studies propose alternative optimization schemes for the predictor and rejector components \citep{okati2021differentiable}, or offer methods to ameliorate issues like underfitting in surrogate losses \citep{liu2024mitigating,narasimhanpost}. Frameworks for unified post-processing in multi-objective L2D have also been developed, drawing on principles like the generalized Neyman-Pearson Lemma \citep{Mohammad2024,neyman1933ix}. Concurrently, theoretical advancements continue to shed light on L2D mechanisms; examples include the introduction of an asymmetric softmax function to derive more robust probability estimates \citep{cao2023defense}, and investigations into dependent Bayes optimality to better understand the interdependencies in deferral decisions \citep{weiexploiting}. The practical impact of these collective advancements is evident in the successful application of L2D and its variants across a range of domains, including regression tasks, human-in-the-loop systems, and reinforcement learning scenarios \citep{chen2024learning,de2020regression,de2021classification,gao2021human,hemmer2023learning,joshi2021pre,mozannar2022teaching,palomba2024causal,straitouri2021reinforcement,zhao2021directing}. These contributions, while advancing the field, typically engage with models or data where expert interactions are already defined or have occurred, rather than focusing on the strategic acquisition of costly expert knowledge during the primary training phase.

Separately, another avenue of studies address deferral decisions under operational constraints or workload considerations, often with a focus on the inference phase \citep{Mohammad2024,alvescost,zhang2024coverage}. While this line of research is crucial for using L2D systems effectively, its foundational assumptions differ markedly from our work.  Such studies generally proceed with the assumption that expert query costs are predetermined and fully known during the training stage. In direct contrast, our research is grounded in a budgeted deferral setting. The central objective here is to train deferral algorithms that not only learn the deferral policy but do so while actively minimizing the expert querying expenses that are incurred and accounted for as an integral component of the learning process itself, without prior knowledge of these costs.

The most relevant prior work to our study is by
\citet{reid2024online}, who model deferral to a single expert as a
two-armed contextual bandit problem with budget constraints. This
formulation is interesting and allows for the direct application of
existing bandit algorithms with knapsack constraints
\citep{agrawal2016linear, filippi2010parametric,
  li2017provably}. However, extending these methods to multiple
experts is non-trivial. Moreover, such general bandit algorithms
\citep{lattimore2020bandit} are not tailored to the deferral loss or
its consistent surrogate losses, which are central to
learning-to-defer approaches. Importantly, \citet{reid2024online}
assume a generalized linear model for the expert and model performance
\citep{li2017provably}, an assumption that is often too restrictive
for practical deferral scenarios.

\section{Relevance of Bandit Algorithms}
\label{app:bandits}

While multi-armed bandit (MAB) algorithms are powerful tools, our
problem does not naturally fit into this framework. This section
clarifies why both standard and contextual bandit formulations are
misaligned with our setting.

\paragraph{Reward definition mismatch.} 
Our objective depends jointly on the computational cost of querying an
expert and the benefit of using that information to improve the
generalization of the deferral algorithm. It is unclear how such a
two-fold objective could be meaningfully expressed as a single reward
for the “action” (choice of expert $k$) in a bandit framework.

\paragraph{Ambiguous benchmark.} 
The standard contextual bandit benchmark, the best policy in
hindsight is ill-defined here, since our guarantees concern both
generalization and label complexity. A single “best policy” does not
capture both objectives simultaneously.

\paragraph{Mismatch of guarantees.} 
Even if one could design a meaningful bandit formulation, the resulting
algorithms would not directly yield the dual guarantees we seek:
(i) strong generalization for the deferral predictor, and
(ii) favorable label complexity. These guarantees are central to our
analysis and theoretical contributions.

\paragraph{Adversarial bandits.} 
In adversarial settings (e.g., EXP3), treating each expert as an arm
leads to regret benchmarks such as external or shifting regret, which
are either uninformative or inadequate for deferral. More refined
notions such as policy regret or contextual regret also face challenges,
since the reward definition remains ambiguous, and the quality of the
benchmark depends heavily on the choice of policy class. Moreover, the
trade-off between the short-term cost of querying and the long-term
benefit of acquired labels is not captured in these formulations.

\paragraph{Prior work.} 
\citet{reid2024online} formulate deferral to a single expert as a
two-armed contextual bandit problem with budget constraints, enabling
the use of general bandit algorithms with knapsack constraints
\citep{agrawal2016linear, filippi2010parametric, li2017provably}.
However, extending such methods to multiple experts is non-trivial and
remains unexplored in that framework. More importantly, these
general-purpose algorithms are not designed to minimize deferral loss or
to leverage consistent surrogate losses, both of which are central to
our framework. Furthermore, the generalized linear model assumption made
by \citet{reid2024online} can be overly restrictive in practical
scenarios involving black-box experts (e.g., large language models).

\paragraph{Summary.} 
For these reasons, it is difficult to cast our problem as a standard
bandit problem while still achieving the type of dual guarantees, generalization and label complexity bounds, that our analysis provides.

\subsection{Comparison with Contextual Bandit Frameworks}
\label{sec:bandits-comparison}

While the problem of budgeted deferral involves sequential,
context-dependent decisions, framing it within a standard contextual
bandit setting is problematic. The core reason is a fundamental
mismatch between the objectives and theoretical guarantees of each
framework. Our approach is designed to solve a pool-based active
learning problem, whereas contextual bandits are designed for online
learning and performance optimization. We detail the key distinctions
below.

\subsubsection{Mismatch 1: Primary Objective and Performance Benchmark}
The primary objective of our framework is to train a final, static
routing function $r_T$ that exhibits low \emph{generalization
  error}, $\sE(r_T)$, on unseen data drawn from the true
distribution $\sD$. The querying strategy is a mechanism to
gather the most \textit{informative} labels to train this high-quality
global model efficiently. Our performance is benchmarked against the
best possible router in the hypothesis class,
$r^* = \argmin_{r \in \sR} \sE(r)$, and our guarantees
bound the excess error $\sE(r_T) - \sE(r^*)$.

In contrast, the objective of a contextual bandit algorithm is to
maximize the \emph{cumulative immediate reward} over a sequence of $T$
rounds. Its performance is measured by \emph{regret}, which compares
its total reward to that of the best policy $\pi^*$ from a fixed
policy class $\Pi$ in hindsight:
$$
\Reg_T = \sum_{t = 1}^T R(x_t, \pi^*(x_t)) - \sum_{t = 1}^T R(x_t, a_t),
$$
where $a_t$ is the action (expert) chosen at round $t$. A low-regret
algorithm is effective at online decision-making but provides no
direct guarantee on the generalization performance of a model trained
from its experiences. An algorithm could achieve low regret by
exploiting the single best expert for a given context, thereby failing
to gather the diverse data needed to learn the more nuanced decision
boundaries of the globally optimal router $r^*$.

\subsubsection{Mismatch 2: Label Complexity and the Nature of Queries}
Our framework explicitly addresses the dual goal of achieving low
generalization error while minimizing the cost of supervision. This is
reflected in our \emph{label complexity} bounds.

\begin{itemize}
\item \emph{Active Querying Framework (Ours):} The decision to query
  an expert is an integral part of the algorithm. The total number of
  queries, $\sum_{t = 1}^T Q_{t, k_t}$, is a random variable that our
  analysis proves to be small, for instance, sublinear in $T$
  ($O(\log T)$ in the realizable case). The algorithm learns
  effectively even from examples where it makes no query, by using
  them to prune the version space.

\item \emph{Contextual Bandit Framework:} In a standard bandit
  setting, the agent \textit{must} pull an arm (i.e., query an expert)
  at every timestep $t = 1, \dots, T$. The total number of queries is
  therefore always $T$. There is no built-in mechanism for cost-free
  learning. While one could introduce a ``do not query'' arm, defining
  its reward is non-trivial, as it involves an unknown
  \textit{opportunity cost} related to the information lost for
  training the final model, a concept that standard regret analysis
  does not capture.
\end{itemize}

Our method is fundamentally a cost-sensitive active learning algorithm
designed to build a minimally small, high-quality training set, which
is a different goal from the online performance optimization of
bandits.

\subsubsection{Mismatch 3: The Reward Signal}
A crucial difficulty in applying a bandit model is defining an
immediate reward signal $R_t$ that aligns with the long-term goal of
learning a good router. If we choose an expert $k$ for an example
$(x_t, y_t)$, a natural reward might be its accuracy,
$R_t = 1 - c_k(x_t, y_t)$. However, this myopic reward would merely
encourage the bandit to identify the most accurate expert for each
context.

The true ``value'' of a query in our setting is the \emph{information
  gain} it provides for reducing the size of the version space
$\sR_t$. A query is valuable if it reveals high disagreement
among the surviving hypotheses, even if the queried expert is not the
most accurate one for that instance. This notion of value is internal
to the state of the learning algorithm and cannot be expressed as a
simple, external reward signal, making the problem fall outside the
standard contextual bandit model.

In summary, our framework is tailored to the specific problem of
training a deferral model under a budget, providing dual guarantees on
generalization and label complexity that a contextual bandit
formulation cannot naturally replicate.

\section{Boundedness of Slope Asymmetry}
\label{app:slope-asymmetry}

This section shows that the slope asymmetry constant $K_\ell$
(Definition~\ref{def:slope-asymmetry}) is finite for the multinomial
logistic (softmax) loss on compact logit domains, in the same spirit as
the IWAL analysis \citep{beygelzimer2009importance} (see their Lemma~5 and
Corollary~6). Two standard technicalities are made explicit here:
(i) softmax invariance to constant shifts (handled by an identifiability
constraint), and (ii) the existence of a zero-cost expert at each
$(x,y)$ (which prevents the denominator from vanishing).

\paragraph{Setup and identifiability.}
Let
\[
\ell_{\log}(r, x, k)
= - \log \paren*{\frac{e^{r(x,k)}}{\sum_{j=1}^{\num} e^{r(x,j)}}}
= - r(x,k) + \log \sum_{j=1}^{\num} e^{r(x,j)}
\]
be the multinomial logistic loss. Throughout, we assume the logit vector
$r(x,\cdot)$ lies in a compact set and satisfies a standard
\emph{identifiability constraint}, e.g.,
\begin{equation}
\label{eq:identifiability}
r(x, \cdot) \in \cB_B \coloneqq \Bigl\{v \in \Rset^{\num}
\colon \norm{v}_\infty \leq B,  \sum_{j=1}^{\num} v_j = 0 \Bigr\}.
\end{equation}
(Equivalently, one may fix a reference class score to zero; any fixed
gauge that removes the softmax invariance to constant shifts is
acceptable.)

\begin{lemma}[Lipschitz upper bound]
\label{lem:log-lip-upper}
Fix $x$ and $k$. For any $r(x,\cdot), r'(x,\cdot) \in \cB_B$,
\begin{equation*}
\abs*{\ell_{\log}(r, x, k) - \ell_{\log}(r', x, k)}
\leq 2  \norm{r(x,\cdot) - r'(x,\cdot)}_\infty.
\end{equation*}
Consequently,
\begin{equation*}
\sum_{k=1}^{\num}
\abs*{\ell_{\log}(r, x, k) - \ell_{\log}(r', x, k)}
\leq 2  \num  \norm{r(x,\cdot) - r'(x,\cdot)}_\infty.
\end{equation*}
\end{lemma}

\begin{proof}
The gradient with respect to the logit vector has components
\[
\frac{\partial \ell_{\log}}{\partial r(x,k)} = -1 + s_k,
\qquad
\frac{\partial \ell_{\log}}{\partial r(x,j)} = s_j  (j \neq k),
\]
where $s = \mathrm{softmax} \paren*{r(x,\cdot)} \in \Delta^{\num-1}$.
Thus $\norm{\nabla_{r(x,\cdot)} \ell_{\log}}_1
= \abs*{-1 + s_k} + \sum_{j \neq k} s_j = 2 (1 - s_k) \leq 2$.
The mean value inequality yields the first claim; summing over $k$ gives
the second.
\end{proof}

\begin{lemma}[Coercivity on the zero-mean subspace]
\label{lem:coercive}
Let $s \in \Delta^{\num-1}$ and define the linear map
$M_s \coloneqq I - \bone s^\top$. For any $v \in \Rset^\num$ with
$\sum_{j} v_j = 0$,
\begin{equation*}
\norm{M_s v}_\infty \geq \tfrac{1}{2}  \norm{v}_\infty.
\end{equation*}
\end{lemma}

\begin{proof}
Write $(M_s v)_i = v_i - s^\top v$. Since $s^\top v$ is a convex
combination of the coordinates of $v$, it lies in
$\bracket*{\min_j v_j,  \max_j v_j}$. Because $\sum_j v_j = 0$,
we have $\min_j v_j \leq 0 \leq \max_j v_j$, hence
$\max_j v_j - \min_j v_j \geq \norm{v}_\infty$. For any $i$,
\(
\abs{(M_s v)_i}
= \abs{v_i - s^\top v}
\geq \tfrac{1}{2}  \paren*{\max_j v_j - \min_j v_j}
\geq \tfrac{1}{2}  \norm{v}_\infty,
\)
so taking the maximum over $i$ gives the claim.
\end{proof}

\begin{lemma}[Pointwise lower bound on per-class differences]
\label{lem:per-class-lower}
Fix $x$ and consider any $r,r' \in \cB_B$. Then there exists a softmax
vector $s$ on the line segment between $r(x,\cdot)$ and $r'(x,\cdot)$
such that
\[
\ell_{\log}(r', x, \cdot) - \ell_{\log}(r, x, \cdot)
= - \paren*{I - \bone s^\top}  \paren*{r'(x,\cdot) - r(x,\cdot)}.
\]
In particular,
\begin{equation*}
\max_{k} \abs*{\ell_{\log}(r', x, k) - \ell_{\log}(r, x, k)}
\geq \tfrac{1}{2}  \norm{r'(x,\cdot) - r(x,\cdot)}_\infty.
\end{equation*}
\end{lemma}

\begin{proof}
By the mean value theorem applied to the smooth map
$r \mapsto \ell_{\log}(r, x, \cdot)$, there exists
$s \in \Delta^{\num-1}$ on the segment connecting $r$ and $r'$ with
$\nabla \ell_{\log} = I - \bone s^\top$ (row-wise). Since both
$r,r' \in \cB_B$ satisfy $\sum_j r_j = \sum_j r'_j = 0$, their
difference $v = r' - r$ also satisfies $\sum_j v_j = 0$. The identity
then gives $\Delta \ell = - M_s v$ and Lemma~\ref{lem:coercive} yields
$\norm{\Delta \ell}_\infty \geq \tfrac{1}{2} \norm{v}_\infty$.
\end{proof}

\paragraph{Bounded slope asymmetry: non-budgeted case.}
Suppose all experts have zero cost at $(x,y)$ (non-budgeted systems).
Then, for any $r,r' \in \cB_B$ and fixed $x$,
\[
\frac{\max_{\bc} \sum_k (1 - c_k) \abs{\Delta \ell_k}}
{\min_{\bc} \sum_k (1 - c_k) \abs{\Delta \ell_k}}
=
\frac{\sum_{k=1}^{\num} \abs{\Delta \ell_k}}
{\sum_{k=1}^{\num} \abs{\Delta \ell_k}}
= 1,
\]
and trivially $K_{\ell_{\log}} = 1$. More meaningfully, combining
Lemma~\ref{lem:log-lip-upper} and Lemma~\ref{lem:per-class-lower} gives
\begin{equation}
\label{eq:K-nonbudgeted}
\frac{\sum_{k=1}^{\num} \abs{\Delta \ell_k}}
{\max_{k} \abs{\Delta \ell_k}}
\leq \frac{2  \num  \norm{\Delta r}_\infty}
{\tfrac{1}{2} \norm{\Delta r}_\infty}
= 4  \num.
\end{equation}
Thus, on compact domains with the standard identifiability
\eqref{eq:identifiability}, the slope asymmetry is explicitly bounded by
a constant depending only on $\num$ (and the compactness $B$ through
the domain restriction, not through the constant).

\paragraph{Bounded slope asymmetry: budgeted case with zero-cost experts.}
Assume at each $(x,y)$ there exists a nonempty index set
$I(x,y) \subseteq [\num]$ of zero-cost experts, i.e.,
$c_k(x,y) = 0 \iff k \in I(x,y)$. Then
\[
\min_{\bc} \sum_k (1 - c_k) \abs{\Delta \ell_k}
= \sum_{k \in I(x,y)} \abs{\Delta \ell_k},
\quad
\max_{\bc} \sum_k (1 - c_k) \abs{\Delta \ell_k}
= \sum_{k \notin I(x,y)} \abs{\Delta \ell_k} + \sum_{k \in I(x,y)} \abs{\Delta \ell_k}
= \sum_{k=1}^{\num} \abs{\Delta \ell_k}.
\]
Therefore,
\begin{equation}
\label{eq:K-budgeted-master}
K_{\ell_{\log}}
\leq
\frac{\sum_{k=1}^{\num} \abs{\Delta \ell_k}}
{\sum_{k \in I(x,y)} \abs{\Delta \ell_k}}
\leq
\frac{2  \num  \norm{\Delta r}_\infty}
{\sum_{k \in I(x,y)} \abs{\Delta \ell_k}},
\end{equation}
using Lemma~\ref{lem:log-lip-upper} for the numerator.

A simple uniform bound follows if the zero-cost set has size at least a
fixed fraction: suppose there exists $\rho \in (0, 1]$ such that
$\abs{I(x,y)} \geq \rho  \num$ for all $(x,y)$. Then by averaging,
\(
\sum_{k \in I(x,y)} \abs{\Delta \ell_k}
\geq \frac{\abs{I(x,y)}}{\num} \sum_{k=1}^{\num} \abs{\Delta \ell_k}
\geq \rho  \max_k \abs{\Delta \ell_k}.
\)
Combining with Lemma~\ref{lem:per-class-lower} gives
\begin{equation}
\label{eq:K-budgeted-explicit}
K_{\ell_{\log}}
\leq
\frac{2  \num  \norm{\Delta r}_\infty}
{\rho  \max_k \abs{\Delta \ell_k}}
\leq
\frac{2  \num  \norm{\Delta r}_\infty}
{\rho  \tfrac{1}{2} \norm{\Delta r}_\infty}
=
\frac{4}{\rho}  \num.
\end{equation}
Hence, in the budgeted case with a uniform lower bound on the number of
zero-cost experts, $K_{\ell_{\log}}$ is explicitly bounded by
$C(\num,\rho) = \frac{4}{\rho}\num$.

\paragraph{Remarks.}
(i) The identifiability constraint \eqref{eq:identifiability} (or any
equivalent gauge fixing) is standard in multiclass logistic models and
is necessary to avoid the softmax invariance to constant shifts, under
which $\Delta \ell \equiv 0$ can occur for $\Delta r \neq 0$.  
(ii) The constants above are \emph{dimension-explicit} and do not depend
on the particular scores beyond the compactness parameter $B$.  
(iii) The same proof pattern applies to other comp-sum losses whose
per-class losses are Lipschitz in the logits on compact domains (e.g.,
exponential), with constants depending only on $\num$ and the chosen
gauge.  
(iv) When $I(x,y) = [\num]$ (non-budgeted systems), the clean bound
\eqref{eq:K-nonbudgeted} shows $K_{\ell_{\log}} \leq 4 \num$; when
$\abs{I(x,y)} \geq \rho \num$ uniformly, the budgeted bound
\eqref{eq:K-budgeted-explicit} applies.

\section{Proofs of Budgeted Two-Stage Deferral}
\label{app:two-stage}

\subsection{Lemma~\ref{lemma:pair} and Proof}
\label{app:pair}

\begin{restatable}{lemma}{Pair}
\label{lemma:pair} 
For any $\delta > 0$, with probability at least $1 - \delta$, for all
time steps $T$ and all pairs $(r, r') \in \sR_T^2$, we have
\begin{equation*}
\abs*{\sE_T(r) - \sE_T(r') - \paren*{\sE(r) - \sE(r')}} \leq \Delta_T.
\end{equation*}
\end{restatable}

\begin{proof}
Fix any time step $T$ and any pair $(r, r') \in \sR_T^2$.
Define the deviation sequence $Z_t$, $t \in [T]$:
\begin{equation*}
  Z_t = \sum_{k = 1}^{\num} \frac{1_{\{k_t = k\}} Q_{t,k}}{q_{t,k} p_{t,k}} (1 - c_{t,k}(x_t, y_t)) \big( \ell(r, x_t, k) - \ell(r', x_t, k) \big) - \big( \sE(r) - \sE(r') \big).
\end{equation*}
Here, the selection probabilities $p_{t, k}$ depend only on randomness
up to time $t - 1$, and the parameters $q_{t, k}$ are fixed before the
algorithm begins. Thus, $(Z_t)_{t}$ forms a martingale difference
sequence with respect to the natural filtration, since $\E[Z_t \mid \cF_{t-1}] = 0$..

Next, we bound the absolute value of $Z_t$. Using the definition of
$p_{t,k}$ and the boundedness of the loss, we have:
\begin{equation*}
  |Z_t| \leq \frac{1}{q_{\min} p_{t,k}} \abs*{\ell(r, x_t, k) - \ell(r', x_t, k)} + \abs*{\sE(r) - \sE(r')} \leq \frac{1}{q_{\min}} + 1 = \ov q.
\end{equation*}
Applying Azuma’s inequality
\citep[Theorem~D.7]{MohriRostamizadehTalwalkar2018} with failure
probability $\delta / (T(T+1)|\sR|^2)$ gives the desired bound for a
fixed $T$ and pair $(r, r')$. Taking a union bound over all $t \in
[T]$ and all pairs $(r, r') \in \sR_T^2$ yields the
result.
\end{proof}

\subsection{Proof of Theorem~\ref{thm:loss-bound}}
\label{app:loss-bound}

\GeneralizationBound*

\begin{proof}
We prove by induction that $r^* \in \sR_T$ for all $T$. The base case $T = 1$ holds trivially since $\sR_1 = \sR$. Suppose the claim holds for $T$.

Let $r_T$ be the empirical risk minimizer in $\sR_T$. Then, by Lemma~\ref{lemma:pair}:
\begin{equation*}
\sE_{T}(r^{*}) - \sE_{T}(r_{T})
\leq \sE(r^{*}) - \sE(r_{T}) + \Delta_{T}
\leq \Delta_{T}.    
\end{equation*}
Thus $\sE_{T}(r^{*})\leq \sE_{T}^{*} + \Delta_{T}$, implying $r^* \in
\sR_{T + 1}$.

For the second part, let $r, r' \in \sR_T$. Then:
\begin{equation*}
\sE(r) - \sE(r') \leq \sE_{T - 1}(r) - \sE_{T-1}(r') + \Delta_{T - 1}
\leq \sE_{T - 1}^{*} + \Delta_{T - 1} - \sE_{T-1}^{*} + \Delta_{T - 1}
 \leq 2\Delta_{T - 1}.    
\end{equation*}
Applying this to $r_T$ and $r^*$ completes the proof.
\end{proof}

\subsection{Proof of Lemma~\ref{lemma:two-spaces}}
\label{app:two-spaces}

\TwoSpaces*
\begin{proof}
Expanding the definition of $\rho$ and applying the triangle inequality: 
\begin{align*}
\rho(r, r^*)
& =
\E_{(x, y) \sim D} \bracket*{ \max_{\bc \in \curl*{0,1}^{\num}} \sum_{k = 1}^{\num} \paren*{1 - c_k} \abs*{\ell(r, x, k) - \ell(r^*, x, k)}}\\
& \leq
K_{\ell}  \E_{(x, y) \sim D} \bracket*{\sum_{k = 1}^{\num} \paren*{1 - c_k} \abs*{\ell(r, x, k) - \ell(r^*, x, k)}} \\
& \leq
K_{\ell}  \paren*{ \E_{(x, y) \sim D} \bracket*{\sum_{k = 1}^{\num} \paren*{1 - c_k} \ell(r, x, k)} + \E_{(x, y) \sim D} \bracket*{\sum_{k = 1}^{\num} \paren*{1 - c_k} \ell(r^*, x, k)} }\\
& =
K_{\ell}  ( \sE(r) + \sE(r^*) ).
\end{align*}
This completes the proof.
\end{proof}

\subsection{Proof of Theorem~\ref{thm:label}}
\label{app:label}

\LabelBound*
\begin{proof}
Let $r^*$ be the best-in-class minimizer. At time $t$, from Theorem~\ref{thm:loss-bound}, we have \[\sR_t \subset \curl*{r \in \sR\colon \sE(r) \leq \sE^* + 2 \Delta_{t-1} }.\] By Lemma~\ref{lemma:two-spaces}, this implies  $\sR_t \subset B(r^*, \e)$ with $\e = K_{\ell} (2 \sE^* + 2 \Delta_{t-1})$. The expected number of cost queries at round $t$ is:
\begin{align*}
\E \bracket*{\sum_{k = 1}^{\num} 1_{k_t = k} Q_{t, k}} &= \E \bracket*{\sum_{k = 1}^{\num} q_{t, k} p_{t, k}}\\
& = \E \bracket*{\sum_{k = 1}^{\num} q_{t, k} \max_{r, r' \in \sR_t} \abs*{\ell(r, x, k) - \ell(r', x, k)}}\\
& \leq 2 \E \bracket*{\sum_{k = 1}^{\num} q_{t, k} \max_{r \in \sR_t} \abs*{\ell(r, x, k) - \ell(r^*, x, k)}}\\
& \leq 2 \E \bracket*{\sum_{k = 1}^{\num} q_{t, k} \max_{r \in B(r^*,\e)} \abs*{\ell(r, x, k) - \ell(r^*, x, k)}}\\
& \leq 2 \theta \e \E \bracket*{\sum_{k = 1}^{\num} q_{t, k}}\\
& = 4 \theta K_{\ell} (\sE^* + \Delta_{t-1})
\end{align*} 
Summing over $t = 1$ to $T$ gives the claimed result.
\end{proof}

\section{Budgeted Single-Stage Multiple-Expert Deferral}
\label{app:single-stage}

We now consider a budgeted single-stage deferral setting, where the
learner can either predict a label directly or defer to one of $\num$
predefined experts. The decision space is expanded to $\ov \sY = [n +
  \num]$, with the original label space $\sY = [n]$ and deferral
actions labeled $n + 1$ through $n + \num$.

Similar to the two-stage setting, we decompose the decision into two parts: (1) deciding whether to defer and, if so, selecting an expert $k$ and (2) determining the probability of
querying $c_k(x_t, y_t)$ once $k$ is chosen. For the query probability $p_{t,k}$, we carefully
design its expression based on the surrogate loss scores for expert
$k$ computed by each routing function in the current version space,
along with the maximum cost in next section. The distinction from the two-stage setting is that the single-stage setting integrates an additional "no deferral" option, characterized by its own query probability, $q_{t,0}$. This formulation allows us to derive strong theoretical guarantees on
the label complexity.

\subsection{Deferral and Surrogate Loss}

Each hypothesis $h \in \sH$ maps $(x, \ov y) \in \sX \times \ov \sY$
to a real-valued score, with predictions made via $\hh(x) =
\argmax_{\ov y} h(x, \ov y)$. The \emph{single-stage deferral loss}
is:
\begin{equation*}
\ldef(h, x, y, \bc)
= \1_{\hh(x)\neq y} \1_{\hh(x)\in [n]}
+ \sum_{k = 1}^{\num} c_k(x, y) \1_{\hh(x) = n + k},
\end{equation*}
where $c_k(x, y)$ reflects the cost of deferring to expert $k$,
typically defined as the expert’s misclassification error
\citep{verma2023learning}, as in the two-stage setting. Direct
minimization is intractable, thus we will use a surrogate loss
proposed by \citet{verma2023learning,MaoMohriZhong2024deferral}:
\begin{equation*}
  \sfL(h, x, y, \bc)
  =  \ell \paren*{h, x, y}
  + \sum_{k = 1}^{\num} \paren*{1 - c_k(x, y)}   \ell\paren*{h, x, n + k},
\end{equation*}
where $\ell$ is multi-class surrogate loss (e.g., logistic). As
before, we assume $\sfL \in [0, 1]$ and define the generalization
error $\sE(h) = \E[\sfL(h, x, y, \bc)]$.

\subsection{Algorithm Overview} 

Algorithm~\ref{alg:IWAL-single} outlines the procedure.  As with the
two-stage setting, for any $t \in [T]$ and $k \in [0, \num]$, we denote
by $q_{t, k}$ the probability of selecting expert $k$ at time $t$ when
$k \neq 0$ and that of making a direct prediction when $k = 0$. The
optimal choice for the value of $q_t = (q_{t,0}, q_{t,1}, \ldots, q_{t, \num})$
is determined in Appendix~\ref{app:label-single}, using our theoretical
bounds.

If deferring, a Bernoulli trial with success probability
$p_{t,k}$, returned by \emph{Sampling-Probs}, determines whether the
cost $c_{t,k}$ is queried. Queried examples are then stored with
corresponding importance weights.

\begin{algorithm}[h]
  \caption{Budgeted Single-Stage Deferral with Multiple Experts (Subroutine $\textsc{Sampling-Probs}$)}
\label{alg:IWAL-single}
\begin{algorithmic}
\STATE \textsc{Initialize} $S_0 \gets \emptyset$;
\FOR{$t = 1$ \TO $T$}
\STATE \textsc{Receive}$(x_t, y_t)$;
\STATE $p_t \gets \textsc{Sampling-Probs}(x_t, y_t, \curl*{x_s, y_s, \bc_s, q_s, k_s, p_s, Q_s\colon 1 \leq s < t})$;
\STATE $k_t \gets \textsc{Sample}(\num + 1, q_{t,k})$;
\IF{$k_t = 0$}
\STATE $S_{t} \gets S_{t-1} \cup \curl*{\paren*{x_{t}, y_{t}, 0, \frac{1}{q_{t, k_t}}}}$;
\ELSE
\STATE $Q_{t, k_t} \gets \textsc{Bernoulli}(p_{t, k_t})$;
\ENDIF
\IF{$Q_{t, k_t} = 1$}
\STATE $c_{t, k_t} \gets \textsc{Query-Cost}(k_t, (x_t, y_t))$
\STATE $S_{t} \gets S_{t-1} \cup \curl*{\paren*{x_{t}, y_{t}, c_{t, k_t}, \frac{1}{q_{t, k_t}  p_{t, k_t}}}}$;
\ELSE
\STATE $S_t \gets S_{t - 1}$;
\ENDIF
\STATE $r_t \gets \argmin_{r \in \sR} \sum_{(x,y,c, w) \in S_{t}} w  (1 - c) 
\ell(r, x, n + k_t) 1_{k_t \neq 0} + w \cdot (1 - c) \cdot \ell(r, x, y) 1_{k_t = 0}$.
\ENDFOR
\end{algorithmic}
\end{algorithm}

Let $\sD$ be a distribution over $\sX \times \sY \times \curl*{0,
  1}^{\num}$.  The generalization error of $h \in \sH$ on $\sD$ is
given by $\sE(h) = \E_{(x, y, \bc) \sim \sD}[\sfL(h, x, y, \bc)] =
\E_{(x, y, \bc) \sim \sD} \bracket*{\ell(h, x, y) + \sum_{k =
    1}^{\num} \paren*{1 - c_k(x, y)} \ell(h, x, n + k)}$.\ignore{  Since
$\sur$ and $\sD$ are always clear from context, we drop them from
notation. } The importance weighted estimate of the generalization
error of at time $T$ is
\begin{equation*}
\sE_{T}(r) = \frac{1}{T} \sum_{t = 1}^{T} \paren*{\frac{1_{k_t = 0}}{q_{t, 0}} \ell(h, x_t, y_t) + \sum_{k = 1}^{\num} \frac{1_{k_t = k} Q_{t, k} }{q_{t, k} p_{t, k}} (1 - c_{t, k}(x_t, y_t)) \ell(h, x_t, n + k)}    
\end{equation*}
where $(k_t, Q_{t, k})$ is as defined in the algorithm.  It is
straightforward to see that $\E[\sE_{T}(r)] = \sE(r)$, with the
expectation taken over all the random variables involved.
\ignore{
The generalization bounds of Theorem~\ref{thm:loss-bound} provide
guarantees for $\sE_{T}(r)$ in this setting as well.
}

\subsection{Sampling-Probs Strategy and Generalization Guarantees}

We apply a shrinking version space strategy
(Algorithm~\ref{alg:threshold-single}) similar to the two-stage
case. Let $\sH_t$ be the version space at time $t$. Using standard
uniform convergence arguments, we prune $\sH_t$ to keep hypotheses
whose empirical loss is within $\Delta_t$ of the minimum $\sE_t^*$:
\begin{equation}
\sH_{t + 1} = \{h \in \sH_t\colon \sE_t(r) \leq \sE_t^* + \Delta_t \}.    
\end{equation}
where $\Delta_t = \sqrt{\ov q^2 (8 / t) \log (2t (t + 1)\abs*{\sH}^{2} / \delta)}$ and $ \ov q = \frac{1}{q_{\min}} + 1$ with $q_{\min} = \min_{k \in 0 \cup [\num]} q_{t, k} > 0$. 

\begin{algorithm}[h]
\caption{Sampling-Probs Subroutine with Past History}
\label{alg:threshold-single}
\begin{algorithmic}
\STATE \textsc{Initialize} $\sH_0 \gets \sH$;
\FOR{$t = 2$ \TO $T$}
\STATE $\sE_{t - 1}(h) \gets \frac{1}{t - 1}\sum_{s = 1}^{t - 1} \paren[\bigg]{\frac{1_{k_s = 0}}{q_{s, 0}} \ell(h, x_s, y_s) + \sum_{k = 1}^{\num} \frac{1_{k_s = k} Q_{s, k} }{q_{s, k} p_{s, k}} (1 - c_{s, k}(x_s, y_s)) \ell(h, x_s, n + k)}$;
\STATE $\sE_{t - 1}^{*}\gets \min_{h \in \sH_{t - 1}} \sE_{t - 1}(h)$;
\STATE $\sH_t \gets \curl*{h \in \sH_{t - 1}\colon
\sE_{t - 1}(h)
\leq \sE_{t - 1}^{*} + \Delta_{t - 1}}$;
\STATE $p_{t, k} \gets \max_{h, h' \in \sH_t} \paren*{\ell(h, x_t, n + k) - \ell(h', x_t, n + k)}$.
\ENDFOR
\end{algorithmic}
\end{algorithm}

The sampling probability $p_{t,k}$ is based on the variability of the expert-specific component of the surrogate loss:
\begin{align*}
p_{t, k} 
= \max_{h, h' \in \sH_t} \paren*{\ell(h, x_t, n + k) - \ell(h', x_t, n + k)}.    
\end{align*}
Leveraging the decomposability of the surrogate
loss, this design allocates the query budget adaptively, prioritizing
experts and instances where the disagreement among remaining
hypotheses is greatest, thus targeting high-uncertainty regions. 
We now establish high-probability performance guarantees for the
predictors output by the algorithm.
\begin{restatable}[\textbf{Single-Stage Generalization Bound}]{theorem}{GeneralizationBoundSingle}
\label{thm:loss-bound-single} 
Let $\sD$ be any distribution over $\sX \times \sY \times
{0,1}^{\num}$, and let $\sH$ be a hypothesis set. Suppose $h^* \in
\sR$ minimizes the expected surrogate loss $\sE(h)$. Then, for any
$\delta > 0$, with probability at least $1 - \delta$, the following
holds for all $T \geq 1$:
\begin{itemize}
\setlength{\itemindent}{-2em}
\item The optimal hypothesis $h^*$ remains in $\sH_T$.

\item For any $h, h' \in \sH_T$, the difference in generalization
  error satisfies $\sE(h) - \sE(h') \leq 2 \Delta_{T - 1}$.
\end{itemize}
In particular, the learned hypothesis $h_T$ at time $T$ satisfies: $\sE(h_T) - \sE(h^*) \leq 2\Delta_{T - 1}$.
\end{restatable}
We need the following lemma for the proof.
\begin{lemma}
\label{lemma:pair-single} 
For all data distributions $\sD$, for
all hypothesis sets $\sH$, for all $\delta > 0$, with probability
at least $1 - \delta$, for all $T$ and all $h, h' \in \sH_{T}$, 
\begin{equation*}
\abs*{\sE_T(h) - \sE_T(h') - \sE(h) + \sE(h')} \leq \Delta_T.    
\end{equation*}
\end{lemma}
\begin{proof}
Fix any time step $T$ and any pair $(h, h') \in \sH_T^2$.
Define the deviation sequence $Z_t$, $t \in [T]$:
\begin{align*}
Z_{t} &= \frac{1_{k_t = 0}}{q_{t, 0}}\paren*{ \ell(h, x_t, y_t) - \ell(h', x_t, y_t) }\\
& + \sum_{k = 1}^{\num} \frac{1_{k_t = k}Q_{t, k}}{q_{t, k} p_{t, k}} (1 - c_{t, k}(x_t, y_t)) \paren*{\ell(h, x_t, n + k) - \ell(h', x_t, n + k)} - (\sE(h) - \sE(h')).
\end{align*}
Here, the selection probabilities $p_{t, k}$ depend only on randomness
up to time $t - 1$, and the parameters $q_{t, k}$ are fixed before the
algorithm begins. Thus, $(Z_t)_{t}$ forms a martingale difference
sequence with respect to the natural filtration, since $\E[Z_t
  \mid \cF_{t-1}] = 0$. Next, we bound the absolute value of $Z_t$. Using the definition of
$p_{t,k}$ and the boundedness of the loss, we have:
\begin{align*}
  |Z_t| &\leq \frac{1}{q_{\min}} \abs*{\ell(h, x_t, y_t) - \ell(h', x_t, y_t)} 1_{k = 0}\\
  & \qquad + \frac{1}{q_{\min} p_{t, k}}\abs*{\ell(h, x_t, n + k) - \ell(h', x_t, n + k)} 1_{k \neq 0} + \abs*{\sE(r) - \sE(r')}\\
  & \leq \frac{1}{q_{\min}} + 1 = \ov q.
\end{align*}
  
Applying Azuma’s inequality
\citep[Theorem~D.7]{MohriRostamizadehTalwalkar2018} with failure
probability $\delta / (T(T+1)|\sH|^2)$, we have
\begin{align*}
& \Pr[\abs*{\sE_{T}(h) - \sE_{T}(h') - \sE(h) + \sE(h')}\geq\Delta_{T}\\
& = \Pr\bracket[\Bigg]{ \abs[\bigg]{ \frac{1}{T} \sum_{t=1}^{T} \paren[\bigg]{\frac{1_{k_t = 0}}{q_{t, 0}}\paren*{ \ell(h, x_t, y_t) - \ell(h', x_t, y_t) }\\
& \quad + \sum_{k = 1}^{\num} \frac{1_{k_t = k}Q_{t, k}}{q_{t, k} p_{t, k}} (1 - c_{t, k}(x_t, y_t)) \paren*{\ell(h, x_t, n + k) - \ell(h', x_t, n + k)}  - \paren*{\sE(h) - \sE(h')} }} \geq\Delta_{T}}\\
 & = \Pr\left[ \left| \sum_{t=1}^{T}Z_{t} \right| \geq T\Delta_{T}\right]
 \leq 2 e^{-\frac{T \Delta_{T}^{2}}{2 \ov q^2}}
= \frac{\delta}{T(T+1)|\sH|^{2}}.
\end{align*}
Since $\sH_T$ is a random subset of $\sH$, it suffices to take a union bound over all $h, h' \in \sH$, and $T$.  A union bound over $T$ finishes the
proof.
\end{proof}
\begin{proof}[Proof of Theorem~\ref{thm:loss-bound-single}] We prove by induction that $h^* \in \sH_T$ for all $T$. The base case $T = 1$ holds trivially since $\sH_1 = \sH$. Suppose the claim holds for $T$.

Let $h_T$ be the empirical risk minimizer in $\sH_T$. Then, by Lemma~\ref{lemma:pair}:
\begin{equation*}
\sE_{T}(h^{*}) - \sE_{T}(h_{T})
\leq \sE(h^{*}) - \sE(h_{T}) + \Delta_{T}
\leq \Delta_{T}.    
\end{equation*}
Thus $\sE_{T}(h^{*})\leq \sE_{T}^{*} + \Delta_{T}$, implying $h^{*}\in \sH_{T + 1}$.

For the second part, let $h, h'\in \sH_T$. Then:
\begin{equation*}
\sE(h) - \sE(h') \leq \sE_{T-1}(h) - \sE_{T-1}(h') + \Delta_{T-1}
\leq \sE_{T-1}^{*} + \Delta_{T-1} - \sE_{T-1}^{*} + \Delta_{T-1}
 = 2\Delta_{T-1}.
\end{equation*}
Applying this to $h_T$ and $h^*$ completes the proof.
\end{proof}

\subsection{Label Complexity}
\label{app:label-single}

We showed that the generalization error of the classifier output by budgeted single-stage deferral (Sampling-Probs) 
is similar to the generalization error
of the classifier chosen passively after seeing all $T$ labels.
How many of those $T$ labels does the active 
learner request?

To derive label complexity guarantees for our algorithm,
we must adapt existing tools and definition in active
learning to our single-stage deferral setting. In particular, we
will define a new notion of slope asymmetry, hypothesis
distance metric, generalized disagreement
coefficient, based on experts' costs $c_k$ and
tailored to our setting.
These tools allow us to demonstrate that our budgeted
single-stage deferral algorithm can achieve a favorable label complexity, in fact
lower than its fully supervised counterpart when the learning problem
is approximately realizable and the disagreement coefficient of the
hypothesis set is not loo large.

We give label complexity upper bounds for a class of multi-class surrogate loss
functions that includes multinomial logistic loss and a class of cost functions that satisfy natural assumptions. We require that the loss function has bounded 
\emph{slope asymmetry}, defined below.

\begin{definition}[Slope Asymmetry for Single-Stage Deferral]
The \emph{slope asymmetry} of a multi-class loss function $\ell: \sH \times \sX \times [n + \num] \to [0,\infty)$ 
is $K_{\ell} = $
\begin{align*}
\sup_{h, h' \in \sH, x \in \sX, y \in \sY} 
\frac{ \abs*{\ell(h, x, y) - \ell(h', x, y)} + \max_{\bc \in \curl*{0,1}^{\num}} \sum_{k = 1}^{\num} \paren*{1 - c_k} \abs*{\ell(h, x, n + k) - \ell(h', x, n + k)}}{\abs*{\ell(h, x, y) - \ell(h', x, y)} + \min_{\bc \in \curl*{0,1}^{\num}} \sum_{k = 1}^{\num} \paren*{1 - c_k)} \abs*{\ell(h, x, n + k) - \ell(h', x, n + k)}}.  
\end{align*}
\end{definition}
As in the two-stage case,
this quantity is always well-defined and finite if, for every $(x,
y)$, there exists at least one expert $k^*$ with zero cost:
$c_{k^*}(x, y) = 0$. In practice, $K_{\ell}$ is bounded for common convex surrogates such as the logistic loss, provided the range of score functions $r$ is
restricted to a compact interval (e.g., $[-B, B]$).
Next, we define a distance measure over the
hypothesis set that reflects variability in expert-specific loss
components.

\begin{definition}[Single-Stage Hypothesis Distance Metric]
For any $h,h' \in \sH$, define $\rho(h, h') =$
\begin{align*}
\E_{(x, y) \sim D} \bracket*{ \abs*{\ell(h, x, y) - \ell(h', x, y)} + \max_{\bc \in \curl*{0,1}^{\num}}  \sum_{k = 1}^{\num} \paren*{1 - c_k} \abs*{\ell(h, x, n + k) - \ell(h', x, n + k)}}.
\end{align*}
We define the $\e$-ball around $h$ as $B(h, \e) = \curl*{h' \in \sH \colon \rho(h, h') \leq \e }$.
\end{definition}
Suppose $h^* \in \sH$ minimizes the expected surrogate loss: $\sE^* =
\sE(h^*) = \inf_{h \in \sH} \sE(h)$. At time $t$, the version space
$\sH_t$ contains only hypotheses with generalization error at most
$\sE^* + 2\Delta_{t-1}$. But how close are these hypotheses to $h^*$
in $\rho$-distance? The following lemma provides an upper bound in
terms of the slope asymmetry:

\begin{lemma}
For any distribution $\sD$ and any multi-class loss function $\ell$, we have $\rho(h, h^*) \leq K_{\ell} \cdot ( \sE(h) + \sE^* )$ for all $h \in \sH$.
\label{lemma:two-spaces-single}
\end{lemma}
\begin{proof}
Expanding the definition of $\rho$ and applying the triangle inequality: 
\begin{align*}
& \rho(h, h^*) \\
& =  \E_{(x, y) \sim D} \bracket*{ \abs*{\ell(h, x, y) - \ell(h^*, x, y)} + \max_{\bc \in \curl*{0,1}^{\num}}  \sum_{k = 1}^{\num} \paren*{1 - c_k} \abs*{\ell(h, x, n + k) - \ell(h^*, x, n + k)}}\\
& \leq
K_{\ell}  \E_{(x, y) \sim D} \bracket*{ \abs*{\ell(h, x, y) - \ell(h^*, x, y)} + \sum_{k = 1}^{\num} \paren*{1 - c_k} \abs*{\ell(h, x, n + k) - \ell(h^*, x, n + k)}} \\
& \leq
K_{\ell}  \paren[\bigg]{ \E_{(x, y) \sim D} \bracket*{\ell(h, x, y) + \sum_{k = 1}^{\num} \paren*{1 - c_k} \ell(h, x, n + k)}\\
& \qquad + \E_{(x, y) \sim D} \bracket*{\ell(h, x, y) + \sum_{k = 1}^{\num} \paren*{1 - c_k} \ell(h^*, x, n + k)} }\\
& =
K_{\ell}  ( \sE(r) + \sE(r^*) ).
\end{align*}
This completes the proof.
\end{proof}

The following extends the notion of disagreement \citep{hanneke2007bound} to the single-stage deferral setting.

\begin{definition}
The \emph{disagreement coefficient} $\theta$ is the smallest value 
such that, for all $\e > 0$,
\begin{equation*}
\E_{(x, y) \sim D} \sup_{h \in B(h^*, \e)} \sup_{k \in [\num]} \abs*{\ell(h, x, n + k) - \ell(h^*, x, n + k)} \leq  \theta \e.   
\end{equation*}
\end{definition}

We now present an upper bound on the expected number of cost queries
required by the algorithm.

\begin{restatable}[\textbf{Single-Stage Label Complexity Bound}]{theorem}{LabelBoundSingle}
\label{thm:label-single} 
Let $\sD$ be a single-stage deferral distribution and $\sH$ a
hypothesis set. Suppose the loss function $\ell$ has slope asymmetry
$K_{\ell}$ and the disagreement coefficient of the problem is
$\theta$. Then, with probability at least $1 - \delta$, the expected
number of cost queries made by the budgeted single-stage deferral
algorithm over $T$ rounds is bounded by:
\begin{equation}
4\theta \cdot K_{\ell} \cdot 
\paren*{\sE^* T + O\paren*{ \paren*{1 / q_{\min} + 1} \sqrt{T \log (\abs*{\sH} T/\delta)} } },    
\end{equation}
where the expectation is taken over the algorithm's randomness.
\end{restatable}
\begin{proof}
Let $h^*$ be the best-in-class minimizer. At time $t$, from Theorem~\ref{thm:loss-bound-single}, we have \[\sH_t \subset \curl*{h \in \sH \colon \sE(h) \leq \sE^* + 2 \Delta_{t-1} }.\] By Lemma~\ref{lemma:two-spaces-single}, this implies  $\sH_t \subset B(h^*, \e)$  with $\e = K_{\ell} (2 \sE^* + 2 \Delta_{t-1})$. The expected number of cost queries at round $t$ is:
\begin{align*}
\E \bracket*{\sum_{k = 1}^{\num} 1_{k_t = k} Q_{t, k}} &= \E \bracket*{\sum_{k = 1}^{\num} q_{t, k} p_{t, k}}\\
& = \E \bracket*{\sum_{k = 1}^{\num} q_{t, k} \max_{h, h' \in \sH_t} \abs*{\ell(h, x, n + k) - \ell(h', x, n + k)}}\\
& \leq 2 \E \bracket*{\sum_{k = 1}^{\num} q_{t, k} \max_{h \in \sH_t} \abs*{\ell(h, x, n + k) - \ell(h^*, x, n + k)}}\\
& \leq 2 \E \bracket*{\sum_{k = 1}^{\num} q_{t, k} \max_{h \in B(h^*,\e)} \abs*{\ell(h, x, n + k) - \ell(h^*, x, n + k)}}\\
& \leq 2 \theta \e \E \bracket*{\sum_{k = 1}^{\num} q_{t, k}}\\
& \leq 4 \theta K_{\ell} (\sE^* + \Delta_{t-1})
\end{align*} 
Summing over $t = 1$ to $T$ gives the claimed result.
\end{proof}

The theorem establishes a label complexity bound for our budgeted
single-stage deferral algorithm. As in the two-stage setting, in the
realizable case, the bound scales as $\wt O(\sqrt{T})$, significantly
improving over the linear label complexity $\num T$ incurred by
standard two-stage methods that query all expert costs and, even in
the agnostic setting, the bound remains favorable when the optimal
surrogate loss $\sE^*$ is small. For simplicity of exposition, the main body presents this square-root bound, while the stronger logarithmic bound (derived via Freedman’s inequality) is given in full detail in Appendix~\ref{app:enhanced}.

The dependence on the generalized disagreement coefficient is, in
general, unavoidable, as shown by \citet{Hanneke2014}. However, this
coefficient has been shown to be bounded for many common hypothesis
classes, enabling meaningful guarantees in practice.

\paragraph{Optimal \texorpdfstring{$q_t$}{q}}
Finally, we note that both the generalization
and label complexity bounds are minimized when the expert sampling probabilities are uniform:
$q_{t, 0} = q_{t, 1} = \cdots = q_{t, \num} = \frac1{\num + 1}$, for all $t$,
yielding $q_{\min} = 1/(\num + 1)$ since all experts are treated symmetrically. Under this setting, the bounds simplify
to:
\begin{align}
\sE(r_T) & \leq \sE(r^*) + 2(\num + 2) \sqrt{(8 / (T - 1)) \log (2(T - 1) T \abs*{\sH}^{2} / \delta)} \nonumber \\
\E \bracket*{\sum_{t = 1}^T \sum_{k = 1}^{\num} 1_{k_t = k} Q_{t, k}} & \leq 4\theta \cdot K_{\ell} \cdot 
\paren*{\sE^* T + O\paren*{ \paren*{\num + 2} \sqrt{T \log (|H| T/\delta)} } },  
\end{align}
both of which depend on the number of experts. Note that the label complexity is significantly more favorable than in the standard deferral setting, where it is $\num T$.

By leveraging Freedman’s
inequality \citep{freedman1975tail} in place of Azuma’s \citep[Theorem~D.7]{MohriRostamizadehTalwalkar2018} in our analysis, we can in fact derive
learning and sample complexity bounds that depend only logarithmically
on $T$ in the realizable case (see Appendix~\ref{app:enhanced}). This significantly strengthens our
theoretical guarantees and further highlights the advantages of our
approach.

\subsection{Practical Implementation}
\label{app:easy-alg-single}

As in the two-stage setting, our single-stage algorithm admits an
efficient implementation in common scenarios, particularly when the
hypothesis set $\sH$ is convex and the surrogate loss is a convex
comp-sum loss (e.g., logistic loss). Each round then reduces to solving
convex programs for empirical risk minimization and for estimating
sampling probabilities, both of which are tractable using standard
optimization solvers. The approach also extends to more expressive
models such as neural networks, though at the cost of non-convex
optimization.

\paragraph{Comp-sum losses.}
We consider \emph{comp-sum losses} \citep{mao2023cross} as the
multi-class surrogate loss family $\ell$, which includes many popular
losses such as the multinomial logistic loss. A comp-sum loss is defined
for any $(h, x, \ov y) \in \sH \times \sX \times [n + \num]$ as
\begin{equation*}
\ell_{\mathrm{comp}}(h, x, \ov y)
= \Psi \paren*{\frac{e^{h(x, \ov y)}}{\sum_{y' \in [n + \num]} e^{h(x, y')}}},
\end{equation*}
where $\Psi \colon [0, 1] \to \Rset_{+} \cup \curl{+\infty}$ is
non-increasing. A notable instance is $\Psi(u) = - \log u$, which yields
the \emph{multinomial logistic loss}
\citep{Verhulst1838,Verhulst1845,Berkson1944,Berkson1951}.
For suitable choices of $\Psi$, $\ell_{\mathrm{comp}}$ is convex in $h$.

\paragraph{Convex feasible region.}
At each round $t$, Algorithm~\ref{alg:threshold-single} requires solving
two optimization problems over the restricted hypothesis set $\sH_t$,
defined as the intersection of convex constraints accumulated up to
round $t$:
\begin{align*}
& \sH_t = \bigcap_{t' < t}  \curl[\Bigg]{h \in \sH \colon \\
& \frac{1}{t'}
\sum_{i = 1}^{t'} \Biggl(
\frac{1_{k_i = 0}}{q_{i, 0}}  \ell_{\mathrm{comp}}(h, x_i, y_i)
+ \sum_{k = 1}^{\num}
\frac{1_{k_i = k} Q_{i, k}}{q_{i, k} p_{i, k}} 
\paren*{1 - c_{i, k}(x_i, y_i)} 
\ell_{\mathrm{comp}}(h, x_i, n + k)
\Biggr)
\leq \sE_{t'}^{*} + \Delta_{t'} }.
\end{align*}
Since $\ell_{\mathrm{comp}}$ is convex in $h$, each constraint defines
a convex set, and thus $\sH_t$ is convex.

\paragraph{First optimization.}
The first optimization at round $T$ computes the minimal empirical loss:
\begin{align*}
& \sE_{t}^{*} = \min_{h \in \sH_t} \\
& \frac{1}{t} \sum_{i = 1}^{t} \Biggl(
\frac{1_{k_i = 0}}{q_{i, 0}}
 \Psi \paren*{\frac{e^{h(x_i, y_i)}}{\sum_{\ov y \in [n + \num]} e^{h(x_i, \ov y)}}}
+ \sum_{k = 1}^{\num}
\frac{1_{k_i = k} Q_{i, k}}{q_{i, k} p_{i, k}} 
\paren*{1 - c_{i, k}(x_i, y_i)}
 \Psi \paren*{\frac{e^{h(x_i, n + k)}}{\sum_{\ov y \in [n + \num]} e^{h(x_i, \ov y)}}}
\Biggr),
\end{align*}
a convex program in $h$ over the feasible region $\sH_t$.

\paragraph{Second optimization.}
The second optimization determines the sampling probability $p_{t, k}$
for each expert $k$ by maximizing the difference in surrogate losses:
\begin{equation*}
\max_{h, h' \in \sH_t}
\Biggl\{
\Psi \paren*{\frac{e^{h(x, n + k)}}{\sum_{\ov y \in [n + \num]} e^{h(x, \ov y)}}}
- \Psi \paren*{\frac{e^{h'(x, n + k)}}{\sum_{\ov y \in [n + \num]} e^{h'(x, \ov y)}}}
\Biggr\}.
\end{equation*}
Since $\Psi$ is non-increasing, this expression is maximized when one
term is minimized and the other maximized. Define
\begin{align*}
S_{\min}(x, k) &\equiv \min_{h \in \sH_t}
\frac{e^{h(x, n + k)}}{\sum_{\ov y \in [n + \num]} e^{h(x, \ov y)}}, \\
S_{\max}(x, k) &\equiv \max_{h \in \sH_t}
\frac{e^{h(x, n + k)}}{\sum_{\ov y \in [n + \num]} e^{h(x, \ov y)}}.
\end{align*}
Then the optimal loss variation equals
$\Psi\paren*{S_{\min}(x, k)} - \Psi\paren*{S_{\max}(x, k)}$.

\paragraph{Interpretation and extensions.}
This shows that for convex comp-sum losses and convex hypothesis sets
$\sH$, both optimization problems required by
Algorithm~\ref{alg:threshold-single} are convex and can be solved
efficiently. As an example, when $\sH$ is the linear class
$h(x, \ov y) = \bw \cdot \Phi(x, \ov y)$ with $\norm{\bw} \leq B$, the
optimizations reduce to convex programs in $\bw$. For neural networks,
the problems are no longer convex but can be tackled in practice with
standard stochastic gradient descent (SGD).

\paragraph{Heuristics.}
As in the two-stage case (see also IWAL \citep{cortes2019rbal}),
practical heuristics can be used to simplify implementation. For the
first optimization, one can minimize over $\sH$ instead of $\sH_t$, and
for the second optimization, it suffices to impose only the most recent
constraint (from round $t-1$) rather than all past constraints. With
these choices, the optimal solution remains in the feasible set, while
computational cost is reduced.

\section{Improved Sample Complexity and Label Complexity Bounds}
\label{app:enhanced}

By leveraging Freedman’s
inequality \citep{freedman1975tail} in place of Azuma’s \citep[Theorem~D.7]{MohriRostamizadehTalwalkar2018} in our analysis, we can in fact derive
learning and sample complexity bounds that depend only logarithmically
on $T$ in the realizable case. Our derivation primarily follows the approach of \citet{cortes2019rbal}. For simplicity, we present these results for the two-stage setting; the results and proofs for the single-stage setting are essentially analogous.

Fix any time step $T$ and any pair $(r, r') \in \sR_T^2$.
Define the deviation sequence $Z_t$, $t \in [T]$:
\begin{equation*}
  Z_t = \sum_{k = 1}^{\num} \frac{1_{\{k_t = k\}} Q_{t,k}}{q_{t,k} p_{t,k}} (1 - c_{t,k}(x_t, y_t)) \big( \ell(r, x_t, k) - \ell(r', x_t, k) \big) - \big( \sE(r) - \sE(r') \big).
\end{equation*}

Let $q_{\min} = \min_{k \in [\num]} q_{t, k} > 0$. The following result is adapted from \citep[Lemma~3]{kakade2008generalization}, which is derived from \citep{freedman1975tail}.

\begin{restatable}{lemma}{Martingale}
\label{lemma:martingale}
For any $\delta > 0$, with probability at least $1 - \delta$, for all
time steps $T \geq 3$ and all pairs $(r, r') \in \sR_T^2$, we have
\begin{equation*}
\abs*{\sum_{t = 1}^T Z_t} \leq \max \curl*{2 \sqrt{\sum_{t = 1}^T \sum_{k = 1}^{\num} \E \bracket*{\frac{p_{t, k}}{q_{\min}^2} \mid \cF_{t - 1}}}, 6 \sqrt{\log \paren*{\frac{8 \log(T)}{\delta}}}} \times \sqrt{\log \paren*{\frac{8 \log(T)}{\delta}}}.
\end{equation*}
\end{restatable}

\begin{proof}
We apply \citet[Lemma~3]{kakade2008generalization} and the fact that $\abs*{Z_t} \leq  \ov q$. Furthermore, 
\begin{align*}
\var \bracket*{Z_t \mid \cF_{t - 1}} &= \var \bracket*{\sum_{k = 1}^{\num} \frac{1_{\{k_t = k\}} Q_{t,k}}{q_{t,k} p_{t,k}} (1 - c_{t,k}(x_t, y_t)) \big( \ell(r, x_t, k) - \ell(r', x_t, k) \big) \mid \cF_{t - 1}}\\
& = \sum_{k = 1}^{\num} \var \bracket*{\frac{1_{\{k_t = k\}} Q_{t,k}}{q_{t,k} p_{t,k}} (1 - c_{t,k}(x_t, y_t)) \big( \ell(r, x_t, k) - \ell(r', x_t, k) \big) \mid \cF_{t - 1}}\\
&\leq \sum_{k = 1}^{\num} \E \bracket*{\frac{Q_{t, k}^2}{q_{t, k}^2 p_{t, k}^2} (1 - c_{t,k}(x_t, y_t))^2 \big( \ell(r, x_t, k) - \ell(r', x_t, k) \big) ^2 \mid \cF_{t - 1}}\\
&\leq \sum_{k = 1}^{\num} \E \bracket*{\frac{Q_{t, k} p_{t, k}^2}{q_{\min}^2 p_{t, k}^2} \mid \cF_{t - 1}}\\
&= \sum_{k = 1}^{\num} \E\bracket*{\frac{Q_{t, k}}{q_{\min}^2} \mid \cF_{t - 1}}\\
&= \sum_{k = 1}^{\num} \E \bracket*{\frac{p_{t, k}}{q_{\min}^2} \mid \cF_{t - 1}}.
\end{align*}
A union bound over $Z_t$ and $-Z_t$ concludes the proof.
\end{proof}
Given Lemma~\ref{lemma:martingale} above, we can adapt \citep[Lemma~3]{beygelzimer2009importance} to using the Berstein-like inequality.
Specifically, let us adopt the following threshold:
\begin{equation*}
\Delta_{T} = \frac{2}{T q_{\min}} \paren*{\sqrt{\sum_{t = 1}^T \sum_{k = 1}^{\num} p_{t, k}} + 6 \sqrt{\log \paren*{\frac{(3 + \num T)T^2}{\delta}}}} \times \sqrt{\log \paren*{\frac{8 T^2 \abs*{\sR}^2 \log(T)}{\delta}}}.
\end{equation*}
We now establish high-probability performance guarantees for the
predictors output by the algorithm.
\begin{restatable}{lemma}{DeltaT}
\label{lemma:DeltaT}
Given any hypothesis set $\sR$, for any $\delta > 0$, with probability at least $1 - \delta$, for all
time steps $T \geq 3$ and all pairs $(r, r') \in \sR_T^2$,
\begin{equation*}
\abs*{\sE_{T}(r) - \sE_{T}(r') - \sE(r) + \sE(r')} \leq \Delta_T.
\end{equation*}
In particular, if we let $r = r^*$ and $r' = r_T$, it follows that 
\begin{equation*}
 \sE(r_T) \leq \sE(r^*) + \Delta_T.
\end{equation*}
\end{restatable}
\begin{proof}
Apply Lemma~\ref{lemma:martingale} to time $T \geq 3$ and any pair $(r, r') \in \sR_{T}^2$, with error probability $\delta / (T^2 \abs*{\sR}^2)$ for round $T$. A union bound over $T \geq 3$ and $(r, r')$ gives, with probability at least $1 - \delta$,
\begin{align}
\label{eq:aux-1}
& \mspace{-14mu} \abs*{\sE_{T}(r) - \sE_{T}(r') -  \sE(r) + \sE(r')} \nonumber\\ 
& \mspace{-18mu} \leq \frac1T \max \curl*{2\sqrt{\sum_{t = 1}^T \sum_{k = 1}^{\num} \E \bracket*{\frac{p_{t, k}}{q_{\min}^2} \mid \cF_{t - 1}}}, 6 \sqrt{\log \paren*{\frac{8 T^2 \abs*{\sR}^2 \log(T)}{\delta}}}} \times \sqrt{\log \paren*{\frac{8 T^2 \abs*{\sR}^2 \log(T)}{\delta}}}.
\mspace{-6mu}
\end{align}
Next, by \citep[Proposition 2]{cesa2008improved}, with probability at least $1 - \delta$, for all $T \geq 3$,
we can write
\begin{align}
\label{eq:aux-2}
\sum_{t = 1}^T \E \bracket*{\sum_{k = 1}^{\num} p_{t,k} \mid \cF_{t - 1}} & \leq \paren*{\sum_{t = 1}^T \sum_{k = 1}^{\num} p_{t,k}} + 36 \log\paren*{\frac{(3 + \sum_{t = 1}^T \sum_{k = 1}^{\num} p_{t, k})T^2}{\delta}}\\
& \qquad + 2 \sqrt{\paren*{\sum_{t = 1}^T \sum_{k = 1}^{\num} p_{t, k}} \log\paren*{\frac{(3 + \sum_{t = 1}^T \sum_{k = 1}^{\num} p_{t, k})T^2}{\delta}} } \nonumber\\
& \leq \paren*{\sqrt{\sum_{t = 1}^T \sum_{k = 1}^{\num} p_{t, k}} + 6 \sqrt{\log \paren*{\frac{\paren*{3 + \num T}T^2}{\delta}}}}^2.
\end{align}
Combining \eqref{eq:aux-1} and \eqref{eq:aux-2}, we get with probability at least $ 1- 2 \delta$, for all $T \geq 3$,
\begin{align*}
&\abs*{\sE_{T}(r) - \sE_{T}(r') -  \sE(r) + \sE(r')}\\ &\leq \frac{2}{T q_{\min}} \paren*{\sqrt{\sum_{t = 1}^T \sum_{k = 1}^{\num} p_{t, k}} + 6 \sqrt{\log \paren*{\frac{(3 + \num T)T^2}{\delta}}}} \times \sqrt{\log \paren*{\frac{8 T^2 \abs*{\sR}^2 \log(T)}{\delta}}},
\end{align*}
as claimed.
\end{proof}

We now present an upper bound on the expected number of cost queries
required by the algorithm. 
\begin{restatable}{lemma}{Label}
\label{lemma:label}
Given any hypothesis set $\sR$, and distribution $\sD$, with $\theta(\sD, \sR) = \theta$, for all $\delta > 0$, for all $T \geq 3$, with probability at least $1 - \delta$, we have
\begin{equation*}
\sum_{t = 1}^T \E \bracket*{\sum_{k = 1}^{\num} 1_{k_t = k} Q_{t, k}  \mid \cF_{t - 1}} \leq 4 \num \theta K_{\sfL} \paren*{\sE^* T + O\paren*{\sqrt{\sE^* T \log(T \abs*{\sR} /\delta)} }} + O\paren*{\log^3 \paren*{T \abs*{\sR} / \delta}},
\end{equation*}
where $K_{\sfL}$ is a constant that depends on the loss function $\sfL$.
\end{restatable}
\begin{proof}
By \citep[Theorem 11]{beygelzimer2009importance}, for $t \geq 3$, the following holds:
\begin{equation*}
\E \bracket*{\sum_{k = 1}^{\num} p_{t,k} \mid \cF_{t - 1}} \leq 4 \num \theta K_{\sfL} \paren*{\sE^* + \Delta_{t - 1}},
\end{equation*}
where $\sE^* =
\sE(r^*) = \inf_{r \in \sR} \sE(r)$ is the error of best-in-class. Plugging in the expression for $\Delta_{t - 1}$, and applying again a similar concentration inequality as before to relate $\sum_{t = 1}^T \sum_{k = 1}^{\num} p_{t, k}$ to $\sum_{t = 1}^T \E \bracket*{\sum_{k = 1}^{\num} p_{t, k} \mid \cF_{t - 1}}$, we end up with a recursion on  $\E \bracket*{ \sum_{k = 1}^{\num} p_{t, k} \mid \cF_{t - 1}}$:
\begin{equation}
\label{eq:aux-3}
\E \bracket*{\sum_{k = 1}^{\num} p_{t, k} \mid \cF_{t - 1}} \leq 4 \num \theta K_{\sfL} \sE^* + \frac{4 \num \theta K_{\sfL} c_1}{t  - 1} \sqrt{\sum_{s = 1}^{t - 1} \E_{(x_t, y_t)} \bracket*{p_s \mid \cF_{s - 1}}} + c_2 \paren*{\frac{\log \bracket*{(t - 1)\abs*{\sR} / \delta}}{t - 1}},
\end{equation}
where $c_1 = 2 \sqrt{\log \paren*{\frac{8 T^2 \abs*{\sR}^2 \log(T)}{\delta} }} = O\paren*{\sqrt{\log \paren*{\frac{T \abs*{\sR}}{\delta}}}}$, and $c_2$ is a constant. For simplicity, denote by $4 \num \theta K_{\sfL} = c_0$. We show by induction that for all $ t \geq 3 $, 
\begin{equation}
\label{eq:aux-4}
\E \bracket*{\sum_{k = 1}^{\num} p_{t, k} \mid \cF_{t - 1}} \leq c_0 \sE^* + c_4 \sqrt{\frac{\sE^*}{t - 1}} + \frac{c_5}{t - 1},
\end{equation}
for some constants $c_4, c_5$. Assume by induction that \eqref{eq:aux-4} holds for all $s \leq t - 1$. Thus, from \eqref{eq:aux-3}, we have 
\begin{align*}
& \E \bracket*{\sum_{k = 1}^{\num} p_{t, k} \mid \cF_{t - 1}}\\
& \leq c_0 \sE^* + \frac{c_0 c_1}{t - 1} \sqrt{c_0 \sE^* (t - 1) + 2 c_4 \sqrt{\sE^* (t - 1)} + c_5 \log(t - 1)} + c_2 \paren*{\frac{\log \bracket*{(t - 1) \abs*{\sR} / \delta}}{t - 1}}\\
& \leq c_0 \sE^* + \frac{c_0 c_1}{t - 1} \bracket*{\sqrt{c_0 \sE^* (t - 1) + 2 c_4 \sqrt{\sE^* (t - 1)}} + \sqrt{c_5 \log(t - 1)}} + c_2 \paren*{\frac{\log \bracket*{(t - 1) \abs*{\sR} / \delta}}{t - 1}}\\
&\leq c_0 \sE^* + \frac{c_0 c_1}{t - 1} \bracket*{\sqrt{c_0 \sE^* (t - 1)} + \frac{c_4}{\sqrt{c_0}}} + \frac{c_0 c_1 \sqrt{c_5 \log(t - 1)} + c_2 \log \bracket*{(t - 1)\abs*{\sR} / \delta}}{t - 1}\\
& = c_0 \sE^* + \frac{c_0 c_1 \sqrt{c_0 \sE^*}}{\sqrt{t - 1}} + \frac{\sqrt{c_0} c_1 c_4 + c_0 c_1 \sqrt{c_5 \log (t - 1)} + c_2 \log \bracket*{(t - 1) \abs*{\sR} / \delta}}{t - 1},
\end{align*}
where we use the fact that $\sqrt{a + b} \leq \sqrt{a} + \frac{b}{2 \sqrt{a}}$ for $a, b > 0$. 
To complete the induction, we need to show that
\begin{equation*}
\frac{c_0 c_1 \sqrt{c_0 \sE^*}}{\sqrt{t - 1}} + \frac{\sqrt{c_0} c_1 c_4 + c_0 c_1 \sqrt{c_5 \log (t - 1)} + c_2 \log \bracket*{(t - 1) \abs*{\sR} / \delta}}{t - 1} \leq c_4 \sqrt{\frac{\sE^*}{t - 1}} + \frac{c_5}{t - 1}.
\end{equation*}
Thus, $c_4 = c_0 c_1 \sqrt{c_0} = O\paren*{\sqrt{\log \paren*{\frac{T \abs*{\sR}}{\delta}}}}$, and 
\begin{align*}
&c_5 \geq c_0^2 c_1^2 + c_0 c_1 \sqrt{c_5 \log (t - 1)} + c_2 \log \bracket*{(t - 1) \abs*{\sR} / \delta} \implies c_5 = O(c_0^2 c_1^2 \log(T)) = O(\log^2 (T \abs*{\sR} / \delta)).
\end{align*}
Thus, 
\begin{align*}
\E \bracket*{\sum_{k = 1}^{\num} p_{t, k} \mid \cF_{t - 1}} \leq c_0 \sE^* + O \paren*{\sqrt{\log \paren*{T \abs*{\sR} / \delta}}} \sqrt{\frac{\sE^*}{(t - 1)}} + \frac{O(\log^2 (T \abs*{\sR} / \delta))}{t - 1}.
\end{align*}
Finally, 
\begin{align}
\label{eq:aux-5}
\sum_{t = 1}^T \E \bracket*{\sum_{k = 1}^{\num} 1_{k_t = k} Q_{t, k} \mid \cF_{t - 1}}
& = \sum_{t = 1}^T \E \bracket*{\sum_{k = 1}^{\num} q_{t, k} p_{t, k} \mid \cF_{t - 1}} \nonumber \\
& \leq \sum_{t = 1}^T \E \bracket*{\sum_{k = 1}^{\num} p_{t, k} \mid \cF_{t - 1}} \nonumber \\
& \leq 4 \num \theta K_{\sfL} \paren*{\sE^* T + O\paren*{\sqrt{\sE^* T \log(T \abs*{\sR} /\delta)} }} + O\paren*{\log^3 \paren*{T \abs*{\sR} / \delta}},
\end{align}
\end{proof}

Finally, we present the improved sample complexity and label complexity bounds for the budgeted two-stage deferral algorithm.
\begin{restatable}{theorem}{Bdef}
\label{Thm:bdef}
Let $r_T$ denote the hypothesis returned by the budgeted two-stage deferral algorithm after $T$ rounds and $\tau_T$ the total number of cost
queries. Then, for all $\delta > 0$, with probability at least $1 - \delta$, for any $T > 0$, the following inequality holds:
\begin{equation*}
\sE(r_T) \leq \sE^* + \frac{2}{T q_{\min}} \paren*{\sqrt{\sum_{t = 1}^T \sum_{k = 1}^{\num} p_{t, k}} + 6 \sqrt{\log \paren*{\frac{(3 + \num T)T^2}{\delta}}}} \times \sqrt{\log \paren*{\frac{8 T^2 \abs*{\sR}^2 \log(T)}{\delta}}}.
\end{equation*}
Moreover, with probability at least $1 - \delta$, for any $T > 0$, the
following inequality holds:
\begin{equation*}
\tau_T \leq 8 \num \theta K_{\sfL} \paren*{\sE^* T + O\paren*{\sqrt{\sE^* T \log(T \abs*{\sR} /\delta)} }} + O\paren*{\log^3 \paren*{T \abs*{\sR} / \delta}},
\end{equation*}
where $K_{\sfL}$ is a constant that depends on the loss function $\sfL$.
\end{restatable}

\begin{proof}
The bound of generalization error $\sE(r_T)$ follows from Lemma~\ref{lemma:DeltaT}. To get the bound on the number of labels $\tau_{T}$, we relate $\sum_{t = 1}^T \E \bracket*{\sum_{k = 1}^{\num} 1_{k_t = k} Q_{t, k} \mid \cF_{t - 1}}$ in Lemma~\ref{lemma:label} to $\tau_{T} = \sum_{t = 1}^T \sum_{k = 1}^{\num} 1_{k_t = k} Q_{t, k}$ through a Bernstein-like inequality for martingales.
Combining with \eqref{eq:aux-5} completes the proof.   
\end{proof}

\paragraph{Optimal $q_t$} We note that the generalization bound is minimized when the expert sampling
probabilities are uniform: $q_{t,k} = 1/\num$ for all $t$ and $k$,
yielding $q_{\min} = 1/\num$ since all experts are treated symmetrically. Thus, the learning and sample complexity bounds in Theorem~\ref{Thm:bdef} depend only logarithmically on $T$ in the realizable case $\sE^* = 0$ and improve upon the bounds given by \eqref{eq:bounds-linear} in Section~\ref{sec:label}.

\section{Budgeted Deferral with \texorpdfstring{$\e$}{E}-Cover}
\label{app:infinite}

Our framework for budgeted deferral algorithms, along with its associated theoretical guarantees, can be effectively extended from finite to infinite hypothesis classes, using covering numbers. $\e$-Covers allow us to approximate an infinite hypothesis set with a carefully chosen finite one, while keeping the approximation error bounded.

Let's first formally define an $\e$-cover.

\begin{definition}
Let $\sR_{\infty}$ be an infinite hypothesis set equipped with a distance metric $\rho(\cdot, \cdot)$.
A subset $\sG \subset \sR_{\infty}$ is termed an $\e$-cover of $\sR_{\infty}$ if, for every hypothesis
$r \in \sR_{\infty}$, there exists a corresponding hypothesis $g \in \sG$ such that their distance $\rho(r, g)$ is no more than $\e$. The covering number, denoted by $\cN\paren*{\sR_{\infty}, \e}$, represents the cardinality of the smallest possible $\e$-cover for $\sR_{\infty}$. 
\end{definition}
The practical use of an $\e$-cover stems from its ability to ensure that the performance achievable within the finite cover $\sG$ is close to the optimal performance within the larger infinite set $\sR_{\infty}$. This relationship is captured by the following well-established lemma (see, e.g., \citep[Lemma~4]{cortes2019disgraph}).
\begin{lemma}
\label{lemma:cover}
The minimum error achievable within an infinite hypothesis set $\sR_{\infty}$ and and the minimum error achievable within its $\e$-cover $\sG$ differ by at most $\e$. Specifically:
\begin{equation*}
\min_{r \in \sR_{\infty}} \sE(r) \leq \min_{g \in \sG} \sE(g) \leq \min_{r \in \sR_{\infty}} \sE(r) + \e.
\end{equation*}
\end{lemma}

\begin{proof}
The first inequality, $\min_{r \in \sR_{\infty}} \sE(r) \leq \min_{g \in \sG} \sE(g)$, is straightforward, as $\sG$ is a subset of $\sR_{\infty}$, meaning the best hypothesis in $\sR_{\infty}$ must be at least as good as the best in $\sG$.

For the second inequality, let $r^* = \argmin_{r \in \sR_{\infty}} \sE(r)$ denote a hypothesis that achieves the minimum error in $\sR_{\infty}$. By the definition of an $\e$-cover, there must exist a hypothesis $g^* \in \sG$ such that $\rho(g^*, r^*) \leq \e$. Assuming a standard property where the absolute difference in errors is bounded by this distance metric, we have:
\begin{equation*}
\abs*{\sE(g^*) - \sE(r^*)} \leq \rho(g^*, r^*) \leq \e.
\end{equation*}
This implies that $\sE(g^*) \leq \sE(r^*) + \e$.
Since $\min_{g \in \sG} \sE(g)$ is, by definition, less than or equal to the error of any specific $g^* \in \sG$, it follows that: $\min_{g \in \sG} \sE(g) \leq \sE(g^*) \leq \sE(r^*) + \e$.
This establishes the second inequality and completes the proof.
\end{proof}
Lemma~\ref{lemma:cover} provides a crucial insight: when the learner is faced with an infinite family of hypotheses $\sR_{\infty}$, using a finite $\e$-cover $\sG$ results in an approximation error (i.e., the potential increase in the minimum achievable error) of at most $\e$. Therefore, our budgeted deferral algorithm can be run using this finite hypothesis set $\sG$ (ideally, the one with the smallest cardinality) as a proxy for $\sR_{\infty}$. This strategy allows the algorithm to achieve favorable learning guarantees, ensuring that its performance is within $\e$ of the optimal performance achievable within the original infinite hypothesis class.

\section{High-Probability Label Complexity Bounds}
\label{app:high-probability}

Our main label complexity guarantees are stated in expectation form, as is standard in much of the prior active learning literature \citep{beygelzimer2009importance,cortes2019disgraph,cortes2019rbal,cortes2020adaptive,MohriRostamizadehTalwalkar2018}.  
Here, we show how they can be strengthened to high-probability bounds.

\begin{theorem}
\label{th:queries}
Fix the sample $S$. Let $Q_t = \sum_{k=1}^{n_e} \Ind_{k_t = k} Q_{t,k} \in \curl*{0,1}$ be the indicator
variable that the algorithm queries at round $t$, with
$$
p_t = \Pr\bracket*{ Q_t = 1 \mid \cF_{t - 1} }, 
\qquad 
Q(S) = \sum_{t=1}^T Q_t, 
\qquad 
\mu = \E[ Q(S) ] = \sum_{t=1}^T p_t.
$$
Then, for any $\delta \in (0, 1/e)$ and $T \geq 3$, we have
\begin{align*}
\Pr \bracket*{ Q(S) - \mu > \max \curl*{ 2 \sqrt{ \mu \log(1/\delta) }, \, 3 \log(1/\delta) } }
\leq 4 \log(T) \, \delta;    
\end{align*}
If $\mu \geq 4 \log(1/\delta)$, then for any $\e \in (0, 1]$,
\begin{align*}
\Pr \bracket*{ Q(S) \geq (1 + \e) \mu } 
\leq 4 \log(T) \exp \paren*{ - \tfrac{ \e^2 \mu }{4} };
\end{align*}
If $\mu < 4 \log(1/\delta)$, then
\begin{align*}
\Pr \bracket*{ Q(S) \geq \mu + 3 \log(1/\delta) }
\leq 4 \log(T) \, \delta.
\end{align*}
\end{theorem}

\begin{proof}
Define $X_t = Q_t - p_t$. Then $\curl*{X_t, \cF_t}$ is a martingale-difference sequence with $\abs*{X_t} \leq 1$ and $\Var(X_t \mid \cF_{t-1}) = p_t (1 - p_t)$. Thus
$$
V = \sum_{t=1}^T p_t (1 - p_t) \leq \mu.
$$
Applying the version of Freedman’s inequality from \citet[Lemma~3]{kakade2008generalization} with $b = 1$ gives the first inequality.  

For the second statement, if $\mu \geq 4 \log(1/\delta)$ then the $2 \sqrt{ V \log(1/\delta) }$ term dominates. Setting $\log(1/\delta) = \e^2 \mu / 4$ yields the multiplicative form.  

For the third statement, the additive regime
  $3 \log(1/\delta)$ dominates, which directly gives the claimed
  bound.
\end{proof}

\begin{corollary}
\label{cor:hp-label}
Assume that with probability at least $1 - \delta_1$ over the draw of $S$ we have
$\mu = \E[Q(S)] \leq B$. Fix $\e > 0$ and let
$$
\delta_2 = 4 \log(T) \exp \paren*{ - \tfrac{ \e^2 B }{4} }.
$$
Then, with probability at least $1 - (\delta_1 + \delta_2)$ (over both $S$
and the algorithm’s coin flips),
$$
Q(S) \leq (1 + \e) B.
$$
\end{corollary}

\begin{proof}
Condition on the event $\mu \leq B$, which occurs with probability at least $1 - \delta_1$. Applying Theorem~\ref{th:queries} with $\mu \leq B$ gives
$$
\Pr \bracket*{ Q(S) \geq (1 + \e) B } \leq \delta_2.
$$
A union bound yields the claim.
\end{proof}

\noindent\textbf{Discussion.}  
Assume a fixed total failure probability $\delta_{\rm{total}} = \delta$. By setting $\delta_1 = \delta/2$ and $\delta_2 = \delta/2$, we can solve for the required bound $B$ on the expected queries $\mu$.  
If the
condition on the expected number of queries holds, $\Pr[ \mu \leq B ] \geq 1 - \delta/2$, and we choose $B$ such that
$$
B \geq \frac{4}{\e^2} \log \bracket*{ \frac{8 \log(T)}{\delta} },
$$
then with an overall probability of at least $1 - \delta$, the total
number of queries $Q(S)$ is bounded by
$$
Q(S) \leq (1 + \e) B.
$$
This shows that if the expected number of queries $\mu$ is logarithmic in both the time horizon $T$ and inverse probability $1/\delta$, the realized number of queries $Q(S)$ is tightly concentrated around its expectation with high probability.

\end{document}